\documentclass[nohyperref]{article}

\usepackage{microtype}
\usepackage{graphicx}
\usepackage{subfigure}
\usepackage{booktabs} %
\usepackage{enumitem}
\usepackage{hyperref}

\usepackage[accepted]{icml2022}

\RequirePackage{xcolor}
\usepackage[utf8]{inputenc}
\usepackage{caption}
\usepackage{mathtools}
\usepackage{amsfonts}
\usepackage{amsmath}
\usepackage{amsthm}
\usepackage{amssymb}
\input{mysymbol.sty}
\usepackage{tikz-cd}
\usepackage{longtable}
\usepackage{makecell}
\usepackage{thm-restate}
\usepackage{thmtools}
\usepackage{lipsum}
\usepackage{thm-restate}
\usepackage{hyperref}
\usepackage[capitalize]{cleveref}
\usepackage{adjustbox}
\usepackage{makecell, multirow}
\usepackage{xspace}
\usepackage{stackrel}

\setcellgapes{2.5pt}
\newcommand\numberthis{\addtocounter{equation}{1}\tag{\theequation}} %

\newcommand{\edgeweight}   {{edge weight continuous model}\xspace}
\newcommand{\edgeprob}  {{edge probability discrete model}\xspace}
\newcommand{\cIGN}  {{cIGN}\xspace}
\newcommand{\textNewnorm}  {{Partition-norm}\xspace}
\newcommand{\textnewnorm}  {{partition-norm}\xspace}
\newcommand{\parspace}  {{\Gamma}}

\icmltitlerunning{Convergence of Invariant Graph Networks}

\begin{document}

\twocolumn[
\icmltitle{Convergence of Invariant Graph Networks}

\icmlsetsymbol{equal}{*}
\begin{icmlauthorlist}
\icmlauthor{Chen Cai}{ucsd}
\icmlauthor{Yusu Wang}{ucsd}
\end{icmlauthorlist}
\icmlaffiliation{ucsd}{University of California San Diego, San Diego, USA}

\icmlcorrespondingauthor{Chen Cai}{c1cai@ucsd.edu}

\icmlkeywords{graph neural networks}

\vskip 0.3in
]

\printAffiliationsAndNotice{} %

\begin{abstract}
Although theoretical properties such as expressive power and over-smoothing of graph neural networks (GNN) have been extensively studied recently, its convergence property is a relatively new direction. In this paper, we investigate the convergence of one powerful GNN, Invariant Graph Network (IGN) over graphs sampled from graphons.

We first prove the stability of linear layers for general $k$-IGN (of order $k$) based on a novel interpretation of linear equivariant layers. Building upon this result, we prove the convergence of $k$-IGN under the model of \citet{ruiz2020graphon}, where we access the edge weight but the convergence error is measured for graphon inputs. %

Under the more natural (and more challenging) setting of \citet{keriven2020convergence} where one can only access 0-1 adjacency matrix sampled according to edge probability, we first show a negative result that the convergence of any IGN is not possible. We then obtain the convergence of a subset of IGNs, denoted as \smallIGN{}, after the edge probability estimation. We show that \smallIGN{} still contains function class rich enough that can approximate \sGNN{}s arbitrarily well. %
Lastly, we perform experiments on various graphon models to verify our statements.

\end{abstract}

\section{Introduction}

Graph neural networks (GNNs) have recently become a key framework for the learning and analysis of graph type of data, leading to progress on link prediction, knowledge graph embedding, and property prediction to name a few \citep{wu2020comprehensive,zhou2020graph}.
Although theoretical properties such as expressive power \citep{maron2019universality,keriven2019universal,maron2019provably,garg2020generalization,azizian2020expressive,geerts2020expressive,bevilacqua2021equivariant} and over-smoothing \citep{li2018deeper,oono2019graph, cai2020note, zhou2021dirichlet} of GNNs have received much attention, their convergence property is less understood. In this paper, we systematically investigate the convergence of one of the most powerful families of GNNs, the \emph{Invariant Graph Network (IGN)} \citep{maron2018invariant}.
Different from message passing neural network (MPNN) \citep{gilmer2017neural}, it treats graphs and associated node/edge features as monolithic tensors and processes them in a permutation equivariant manner. \IGN{} can approximate the message passing neural network (MPNN) arbitrarily well on the compact domain. When allowing the use of high-order tensor as the intermediate representation, $k$-IGN is shown at least as powerful as $k$-WL test. As the tensor order $k$ goes to $\mc{O}(n^4)$, it achieves the universality and can distinguish all graphs of size $n$ \citep{maron2019universality,keriven2019universal,azizian2020expressive}. %

The high level question we are interested in is the convergence and stability of GNNs. In particular, given a sequence of graphs sampled from some generative models, does a GNN performed on them also converge to a limiting object? This problem has been considered recently, however, so far, the studies \citep{ruiz2020graphon,keriven2020convergence} focus on the convergence of \emph{spectral GNNs}, which encompasses several models \citep{bruna2013spectral,defferrard2016convolutional} including GCNs with order-1 filters \citep{kipf2016semi}. However, it is known that the expressive power of GCN is limited.
Given that 2(k)-IGN is strictly more powerful than GCN \citep{xu2018powerful} in terms of separating graphs\footnote{In terms of separating graphs, $\text{$k$-IGN}>\text{$2$-IGN}=\text{GIN}>\text{GCN}$ for $k>2$. } and its ability to achieve universality, it is of great interest to study the convergence of such powerful GNN. In fact, it is posted as an open question in \citet{keriven2021universality} to study convergence for models more powerful than spectral GNNs and higher order GNNs. This is the question we aim to study in this paper.

\textbf{Contributions.}
We present the first convergence study of the powerful $k$-IGNs (strictly more powerful than the Spectral GNN which previous work studied).
We first analyze the building block of IGNs: linear equivariant layers, and develop a stability result for such layers.
The case of $2$-IGN is proved via case analysis while the general case of $k$-IGN uses a novel interpretation of the linear equivariant layers which we believe is of independent interest.

There have been two existing models of convergence of \sGNN{}s for graphs sampled from graphons developed in \citet{ruiz2020graphon} and \citet{keriven2020convergence}, respectively.
Using the model of \citet{ruiz2020graphon} (denoted by the \textit{\edgeweight{}}) where we access the edge weight but the convergence error is measured between \textit{graphon inputs} (see \Cref{sec:convergence-ruiz} for details), we obtain analogous convergence results for $k$-IGNs. %
The results cover both deterministic and random sampling for {\bf $k$-IGN} while \citet{ruiz2020graphon}
only covers deterministic sampling for the much weaker {\bf Spectral GNN}s.

Under more natural (and more challenging) setting of \citet{keriven2020convergence} where one can only access 0-1 adjacency matrix sampled according to edge probability (called the \textit{\edgeprob{}}), we first show a negative result that in general the convergence of all IGNs is not possible.
Building upon our earlier stability result, we obtain the convergence of a subset of IGN, denoted as \smallIGN{}, after a step of edge probability estimation. We show that \smallIGN{} still contains rich function class that can approximate Spectral GNN arbitrarily well. %
Lastly, we perform experiments on various graphon models to verify our statements.

\section{Related Work}
One type of convergence in deep learning concerns the limiting behavior of neural networks when the width goes to infinity \citep{jacot2018neural,du2018gradient,arora2019fine,lee2019wide,du2019gradient}.  In that regime, the gradient flow on a normally initialized, fully connected neural network with a linear output layer in the infinite-width limit turns out to be equivalent to kernel regression with respect to the Neural Tangent Kernel \citep{jacot2018neural}.

Another type of convergence
concerns the limiting behavior of neural networks when the depth goes to infinity. In the continuous limit, models such as residual networks, recurrent neural network decoders, and normalizing flows can be seen as an Euler discretization of an ordinary differential equation \cite{weinan2017proposal, chen2018neural,lu2018beyond,ruthotto2020deep}. %

The type of convergence we consider in this paper concerns when the input objects converge to a limit, does the output of some neural network over such sequence of objects also converge to a limit? In the context of GNNs, such convergence and related notion of stability and transferability have been studied in both graphon  \citep{ruiz2020graphon,keriven2020convergence,gama2020stability,ruiz2021graph} and manifold setting \citet{kostrikov2018surface, levie2021transferability}. In the manifold setting, the analysis is closely related to the literature on convergence of Laplacian operator \citep{xu2004discrete,wardetzky2008convergence, belkin2008discrete,belkin2009constructing,dey2010convergence}. %

Lastly, after ICML 2022 conference it is brought to our attention that the characterization of linear permutation equivariant layers in $k$-IGN bears similarity in \citet{albooyeh2019incidence}. The pooling and broadcasting operations in \citet{albooyeh2019incidence} are the same as what we call the "averaging" and "replication" operations in our paper. This is discussed in details in \Cref{remark:difference}.   

\begin{table*}[htp]
\renewcommand{\arraystretch}{1.5}
\centering
\tiny
\caption{Linear equivariant maps for $\mb{R}^{n \times n} \rightarrow \mb{R}^{n\times n}$ and $\mb{R}^{[0,1]^2} \rightarrow \mb{R}^{[0,1]^2}$. $\one$ is a all-one vector  of size $n\times 1$ and $\mathrm{I}_{u=v}$ is the indicator function. }
\label{tab:R2-R2}
{\makegapedcells %
\begin{tabular}[]{@{}llll@{}}
\toprule
\makecell[l]{Operations} & Discrete & Continuous & Partitions \\ \midrule
\makecell[l]{1-2: The identity and\\ transpose operations} &
\makecell[l]{$T(A) = A$\\ $T(A) = A^T$} &
\makecell[l]{$T(W) = W$ \\ $T(W) = W^T$} &
\makecell[l]{$\{\{1,3\},\{2,4\}\}$ \\ $\{\{1,4\},\{2,3\}\}$}
\\ \hline

3: The diag operation &
\makecell[l]{$T(A) = \text{Diag}(\text{Diag}^{*}(A))$} &
\makecell[l]{$T(W)(u, v) = W(u, v)\mathrm{I}_{u=v}$}  &
\makecell[l]{$\{\{1,2,3,4\}\}$}
\\ \hline

\makecell[l]{4-6: Average of rows replicated \\ on rows/ columns/ diagonal} &
\makecell[l]{$T(A) = \frac{1}{n}A\one\one^T$ \\ $T(A) = \frac{1}{n}\one(A\one)^T$\\ $T(A) = \frac{1}{n}\text{Diag}(A\one)$}
&
\makecell[l]{$T(W)(*, u) = \int W(u, v)dv$\\ $T(W)(u, *) = \int W(u, v)dv$ \\ $T(W)(u, v) = \mathrm{I}_{u=v}\int W(u, v')dv' $}  &
\makecell[l]{$\{\{1,4\},\{2\},\{3\}\}$ \\ $\{\{1,3\},\{2\},\{4\}\}$ \\ $\{\{1,3,4\},\{2\}\}$}
\\\hline

\makecell[l]{7-9: Average of columns replicated \\on rows/ columns/ diagonal } &
\makecell[l]{$T(A) = \frac{1}{n}A^T\one\one^T$ \\ $T(A) = \frac{1}{n}\one(A^T \one)^T$ \\ $T(A) = \frac{1}{n}\text{Diag}(A^T\one).$} &
\makecell[l]{$T(W)(*, v) = \int W(u, v)du$\\ $T(W)(v, *) = \int W(u, v)du$ \\ $T(W)(u, v) = \mathrm{I}_{u=v}\int W(u', v)du' $} &
\makecell[l]{$\{\{1\},\{2,4\},\{3\}\}$ \\ $\{\{1\},\{2,3\},\{4\}\}$ \\ $\{\{1\},\{2,3,4\}\}$}
\\ \hline

\makecell[l]{10-11: Average of all elements \\replicated on all matrix/ diagonal} &
\makecell[l]{$T(A)=\frac{1}{n^2}(\one^T A\one) \cdot \one \one^T$
\\ $T(A) = \frac{1}{n^2}(\one^T A\one) \cdot \text{Diag}(\one).$} &
\makecell[l]{$T(W)(*, *) = \int W(u, v)dudv$\\ $T(W)(u, v) = \mathrm{I}_{u=v}\int W(u', v')du'dv'$ } &
\makecell[l]{$\{\{1\},\{2\},\{3\},\{4\}\}$ \\ $\{\{1\},\{2\},\{3,4\}\}$}
\\ \hline

\makecell[l]{12-13: Average of diagonal elements \\ replicated on all matrix/diagonal} &
\makecell[l]{$T(A) = \frac{1}{n}(\one^T \text{Diag}^{*}(A)) \cdot \one \one^T$\\ $T(A) = \frac{1}{n}(\one^T \text{Diag}^{*}(A)) \cdot \text{Diag}(\one)$} &
\makecell[l]{$T(W)(*, *) = \int \mathrm{I}_{u=v}W(u, v) dudv$\\ $T(W)(u,v) = \mathrm{I}_{u=v}\int W(u', u') du'$ } &
\makecell[l]{$\{\{1,2\},\{3\},\{4\}\}$ \\ $\{\{1,2\},\{3,4\}\}$}
\\ \hline

\makecell[l]{14-15: Replicate diagonal elements \\ on rows/columns} &
\makecell[l]{$T(A) = \text{Diag}^{*}(A)\one^T$ \\ $T(A) = \one\text{Diag}^{*}(A)^T$} &
\makecell[l]{$T(W)(u, v)= W(u, u)$\\ $T(W)(u, v) = W(v, v)$} &
\makecell[l]{$\{\{1,2,4\},\{3\}\}$ \\ $\{\{1,2,3\},\{4\}\}$}
\\
\bottomrule
\end{tabular}
}
\end{table*}
\section{Preliminaries}
\label{sec:pre}
\subsection{Notations}
To talk about convergence/stability, we will consider graphs of different sizes sampled from a generative model. Similar to the earlier work in this direction, the specific general model we consider is a graphon model.

\textbf{Graphons.} A graphon is a bounded, symmetric and measurable function~$W: [0,1]^2 \to [0,1]$. We denote the space of graphon as $\mc{W}$. It can  be intuitively thought of as an undirected weighted graph with an uncountable number of nodes: roughly speaking, given $u_i, u_j \in [0,1]$, we can consider there is an edge $(i,j)$ with weight $W(u_i, u_j)$.
Given a graphon $W$, we can sample {\bf unweighted} graphs of any size from $W$, either in a deterministic or stochastic manner. We defer the definition of the sampling process until we introduce the \edgeweight{} in \Cref{sec:convergence-ruiz} and \edgeprob{} in \Cref{sec:EP-convergence}.

\textbf{Tensor.} Let $[n]$ denote $\{1, ..., n\}$. A tensor $X$ of order $k$, called a \emph{$k$-tensor}, is a map from $[n]^{\otimes k}$ to $\mb{R}^d$. If we specify a name $\name{i}$ for each axis, we then say $X$ is indexed by $(\name{1}, ..., \name{k})$. With slight abuse of notation, we also write that $X \in \dtensor{k}{d}$. %
We refer to $d$ as the \emph{feature dimensions} or the \emph{channel dimensions}. If $d = 1$, then we have a $k$-tensor $\mb{R}^{n^k\times 1}  = \mb{R}^{n^k}$.
Although the name for each axis acts as an identifier and can be given arbitrarily, we will use \textit{set} to name each axis in this paper. For example, given a 3-tensor $X$, we use $\{1\}$ to name the first axis, $\{2\}$ for the second axis, and so on. The benefits of doing so will be clear in \Cref{subsec:stability-of-k-ign}.

\textbf{Partition.} A partition of $[k]$, denoted as $\gamma$, is defined to be a set of disjoint sets $\gamma:=\{\gamma_1, ..., \gamma_s \}$  with $s\leqslant k$ such that the following condition satisfies,
1) for all $i\in [s],  \gamma_{i} \subset [k]$, 2) $\gamma_i \cap \gamma_j = \emptyset, \forall$ $i, j\in [s]$, and 3) $\cup_{i=1}^{s} \gamma_i = [k]$.
We denote the space of all partitions of $[k]$ as $\parspace_k$. %
Its cardinality is called the $k$-th \emph{bell number} $\bell{k}=|\parspace_k|$.

\textbf{Other conventions.}
 By default, we use 2-norm (Frobenius norm) to refer \lnorm\
for all vectors/matrices and \Lnorm\ for functions on $[0, 1]$ and $[0,1]^2$. $\|\cdot\|_2$ or $\|\cdot\|$ denotes the 2 norm for discrete objects while $\|W\|_{L_2}\coloneqq\int \int W(u, v)dudv$ denotes the norm for continuous objects. Similarly, we use $\| \cdot \|_{{\infty}}$ and $\| \cdot \|_{L_\infty}$ to denotes the infinity norm.
When necessary, we use $\|\cdot \|_{L_2([0, 1])}$ to specify the support explicitly. We use $\specnorm{\cdot}$ to denote spectral norm. $\Phi_c$ and $\Phi_d$ refers to the continuous IGN and discrete IGN respectively. We sometimes call a function $f: [0, 1] \rightarrow \mb{R}^d$ a \emph{graphon signal}. Given $A\in \mb{R}^{n^k \times d_1}, B\in \mb{R}^{n^k \times d_2}$, $[A, B]$ is defined to be the concatenation of $A$ and $B$ along feature dimensions, i.e., $[A, B] \in \mb{R}^{n^k \times (d_1 + d_2)}$. See \Cref{table:symbol_notation} in Appendix for the full symbol list.

\subsection{Invariant Graph Network}

\begin{definition}\label{def:ign}
An Invariant Graph Network (IGN) is a function $\Phi: \mb{R}^{n^{2} \times d_{0}} \rightarrow \mb{R}^d$ of the following form:
\begin{equation}\label{eqn:ign}
F=%
h \circ L^{(T)} \circ \sigma \circ \cdots \circ \sigma \circ L^{(1)},
\end{equation}
where each $L^{(t)}$ is a linear equivariant (LE) layer \citep{maron2018invariant} from $\mb{R}^{n^{k_{t-1}} \times d_{t-1}} \text { to } \mb{R}^{n^{k_{t}} \times d_{t}}$ (i.e., mapping a $k_{t-1}$ tensor with $d_{t-1}$ channels to a $k_t$ tensor with $d_t$ channels), $\sigma$ is nonlinear activation function, $h$ is a linear invariant layer from $k_T$-tensor $\mb{R}^{n^{k_T} \times d_{T}} \text { to vector in } \mb{R}^d$. $d_t$ is the channel number, and $k_t$ is tensor order in $t$-th layer. %
\end{definition}

Let $\dg{\cdot}$ be the operator of constructing a diagonal matrix from vector and $\dgdual{\cdot}$ be the operation of extracting a diagonal from a matrix.
Under the IGN framework, we view a graph with $n$ nodes as a $2$-tensor: In particular,
given its adjacency matrix $A_n$ of size $n\times n$ with node features $X_{n}\in \mb{R}^{n\times d_{\text{node}}}$ and edge features $E_{n \times n} \in \mb{R}^{n^2\times d_{\text{edge}}}$, the input of IGN is the concatenation of $[A_n, \text{Diag}(X_n), E_{n\times n}]\in \mb{R}^{n^2\times (1 + d_{\text{node}}+ d_{\text{edge}})}$ along different channels. %
We drop the subscript when there is no confusion.
We use {\bf $2$-IGN} to denote the IGN whose largest tensor order within any intermediate layer is $2$,
while {\bf $k$-IGN} is one whose largest tensor order across all layers is $k$. We use IGN to refer to the general IGN for any order $k$. %

Without loss of generality, we consider input and output tensor to have a single channel. The extension to multiple channels case is presented in \Cref{subsec:IGN-details}.
Consider all linear equivariant maps from $\mb{R}^{n^{\ell}}$ to $\mb{R}^{n^{m}}$, denoted as $\LE_{\ell + m}$. \citet{maron2018invariant} characterizes the basis of the space of $\LE_{\ell, m}$. It turns out that the cardinality of the basis equals to the bell number $\bell{\ell + m}$, thus depending only on the order of input/output tensor and independent from graph size $n$.
As an example, we list a specific basis of the space of \LE{} maps for \IGN{} (thus with tensor order at most $2$) in \Cref{tab:R2-R2,tab:R1-R2,tab:R2-R1} when input/output channel numbers are both 1. Extending the \LE{} layers to multiple input/output channels is straightforward, and can be achieved by parametrizing the \LE{} layers according to indices of input/output channel. See \Cref{remark:IGN} in Appendix.
Note that one difference of the operators in  \Cref{tab:R2-R1,tab:R1-R2,tab:R2-R2} from those given in the original paper is that
here we normalize all operators appropriately w.r.t. the graph size $n$. (This normalization is also in the official implementation of the IGN paper.)
This is necessary when we consider the continuous limiting case.

To talk about convergence, one has to define the continuous analog of IGN for graphons. In \Cref{tab:R2-R2,tab:R2-R1,tab:R1-R2} we extend all \LE{} operators %
defined for graphs to graphons, %
resulting in the continuous analog of $2$-IGN, denoted as 2-\cIGN{} or $\Phi_c$ in the remaining text. Similar operation can be done in general for $k$-IGN as well, where the basis elements for $k$-IGNs will be described in \Cref{subsec:stability-of-k-ign}.
\begin{definition}[$2$-cIGN]\label{def:cign}
By extending all \LE{} layers for $2$-IGN to the graphon case as shown in \Cref{tab:R2-R2,tab:R2-R1,tab:R1-R2}, we can definite the corresponding 2-\cIGN{} %
via \cref{eqn:ign}.
\end{definition}

\section{Stability of Linear Layers in IGN}
\label{sec:stability}
In this section, we first show a stability result for a single linear layer of IGN. That is, given two graphon $W_1, W_2$, we show that if $\|W_1 - W_2\|_{\tn{pn}}$ is small, then the distance between the objects after applying a single \LE{} layer remain close. Here $\| \cdot \|_{\tn{pn}}$ is a \textnewnorm{} that will be introduced in a moment.
Similar statements also hold for the discrete case when the input is a graph. We first describe how to prove stability for $2$-(c)IGN as a warm-up. We then prove it for $k$-(c)IGN, which is significantly more interesting and requires a new interpretation of the elements in a specific basis of the space of \LE{} operators in \citet{maron2018invariant}.

A the general \LE{} layer $T: \mb{R}^{n^\ell} \to \mb{R}^{n^m}$%
can be written as $T = \sum_{\gamma} c_{\gamma}T_{\gamma}$, where $T_{\gamma} \in \mc{B}:=\set{T_{\gamma}|\gamma \in \parspace_{\ell +m}}$ is the basis element of the space of $\LE_{\ell, m}$ and $c_{\gamma}$ are denoted as filter coefficients. Hence proving the stability of $T$ can be reduced to showing the stability for each element in $\mc{B}$, which we focus from now on.

\subsection{Stability of Linear Layers of $2$-IGN}

A natural way to show stability is by showing that the spectral norm of each \LE{} operator in a basis is bounded. However, even for 2-IGN, as we see some \LE{} operator %
requires replicating ``diagonal elements to all rows" (e.g., operator 14-15 in Table \ref{tab:R2-R2}), and has unbounded spectral norm. To address this challenge, we need a more refined analysis. In particular, below we will introduce a ``new" norm that treats the diagonal differently from non-diagonal elements for the $2$-tensor case. We term it \emph{\textnewnorm{}} as later when handling high order $k$-IGN, we will see that this norm arises naturally w.r.t. the partition of index set of tensors.

\begin{restatable}[\textNewnorm{}]{definition}{partitionnormdef}
\label{def:new-norm}
The \emph{\textnewnorm{}} of 2-tensor $A\in \mb{R}^{n^2}$ is defined as $\|A\|_{\tn{pn}}:=(\frac{\|\dgdual{A}\|_2}{\sqrt{n}}, \frac{\|A\|_2}{n})$.
The continuous analog of the \textnewnorm{} for graphon $W \in \mc{W}$ is defined as $\|W\|_{\tn{pn}} = \left(\sqrt{\int W^2(u,u)du}, \sqrt{\iint W^2(u, v) dudv} \right)$.

We refer to the first term as the \emph{normalized diagonal norm} and the second term as the \emph{normalized matrix norm}. Furthermore, we define operations like addition/comparison on the \textnewnorm{} simply as component-wise operations. For example, $\newnorm{A} \le \newnorm{B}$ if each of the two terms of $A$ is at most the corresponding term of $B$.
\end{restatable}

As each term in \textnewnorm{} is a norm on different parts of the input, the \textnewnorm{} is also a norm.
By summing over the finite feature dimension both for finite and infinite cases, the definition of the \textnewnorm{} can be extended to multi-channel tensors $\mb{R}^{n^2\times d}$ and its continuous version $ \mb{R}^{[0,1]^2 \times d}$. See \Cref{subsec:extending-new-norm} for details.  %

The following result shows that each basis operation for 2-IGN, shown in Tables \ref{tab:R2-R2}, \ref{tab:R1-R2} and \ref{tab:R2-R1}, is stable w.r.t. the \textnewnorm{}. Hence a \LE{} layer consisting of a finite combination of these operations will remain stable. The proof is via a case-by-case analysis and can be found in \Cref{app:2-IGN-linear-stability}.

\begin{restatable}[]{proposition}{RtwoRtwo}
\label{prop:R2-R2}
For all \LE{} operators $T_i:\mb{R}^{n^2} \rightarrow \mb{R}^{n^2}$ of discrete $2$-IGN listed in \Cref{tab:R2-R2}, $\|T_i(A)\|_{\tn{pn}} \leqslant  \newnorm{A}$  for any
$A \in \mb{R}^{n^2}$. Similar statements hold for $T_i: \mb{R}^{n} \rightarrow \mb{R}^{n^2}$ and $T_i: \mb{R}^{n^2} \rightarrow \mb{R}^{n}$ in \Cref{tab:R1-R2,tab:R2-R1} in \Cref{app:tables}. In the case of continuous 2-cIGN, the stability also holds. %
\end{restatable}

\begin{remark}
Note that this also implies that given $W_1, W_2 \in \mc{W}$, we have that $\newnorm{T_i(W_1) - T_i(W_2)} \le \newnorm{W_1 - W_2}$. Similarly, given $A_1, A_2 \in \mb{R}^{n^2 \times 1}$ $= \mb{R}^{n^2}$, we have $\newnorm{T_i(A_1) - T_i(A_2)} \le \newnorm{A_1 - A_2}$.
\end{remark}

\subsection{Stability of Linear Layers of \kIGN{}}
\label{subsec:stability-of-k-ign}
We now consider the more general case of \kIGN{}. In principle, the proof of \IGN{} can still be extended to \kIGN{}, but going through all $\bell{k}$ number of elements of \LE{} basis of \kIGN{} one by one
 can be quite cumbersome. %
In the next two subsections, we provide a new interpretation of elements of the basis of space of $\LE_{\ell, m}$ in a unified framework so that we can avoid a case-by-case analysis. Such an interpretation, detailed in \Cref{subsec:interpretation}, is potentially of independent interest.
First, we need some notations. %

\begin{definition}[Equivalence pattern]\label{def:equivpattern}
Given a $k$-tensor $X$, denote the space of its indices $\{(i_1, ..., i_k) \mid i_1\in [n], ..., i_k\in [n]\}$ by $\mc{I}_{k}$. Given $X$, $\gamma =\{\gamma_1, ..., \gamma_d\}\in \parspace_k$ and an element $\bs{a}=(a_1, ..., a_k) \in \mc{I}_{k}$, we say $\bs{a} \in \gamma$ if $i, j \in \gamma_{l}$ for some $l\in[d]$ always implies $a_i = a_j$. %
  Alternatively, we also say $\bs{a}$ satisfies the equivalence pattern of $\gamma$ if $\bs{a} \in \gamma$.
\end{definition}
As an example, suppose $\gamma=\{\{1, 2\}, \{3\}\}$. Then $(x, x, y) \in \gamma$ %
while $(x, y, z) \notin \gamma$.
Equivalence patterns can induce ``slices''/sub-tensors of a tensor.
 \begin{definition}[Slice/sub-tensor of $X \in \dtensor{k}{1}$ for $\gamma \in \parspace_k$]
 \label{def:slice}
 Let $X \in \dtensor{k}{1}$ be a $k$-tensor indexed by $(\set{1}, ..., \set{k})$.
 Consider a partition $\gamma = \{\gamma_1, ..., \gamma_{k'}\} \in \parspace_k$ of cardinality $k' \leqslant k$.
 The \emph{slice (sub-tensor) of $X$ induced by $\gamma$} is a $k'$-tensor $X_{\gamma}$, indexed by $(\gamma_1, ..., \gamma_{k'})$, and defined to be
  $X_{\gamma}(j_1, ..., j_{k'}) := X(\iota_{\gamma}(j_1, ..., j_{k'}))$ where $j_{\cdot} \in [n]$ and $\iota_{\gamma}(j_1, ..., j_{k'})\in \gamma$.
  $\iota_{\gamma}: [n]^{k'} \rightarrow [n]^k$ is defined to be $\iota_{\gamma}(j_1, ..., j_{k'}) := (i_1, ..., i_k) $ such that $\{a, b\}\subseteq \gamma_c$   implies $i_a = i_b := j_c$. Here $a, b\in [k], c\in [k']$. %
As an example, we show five slices of a $3$-tensor in \Cref{fig:slices}.
 \end{definition}
Consider the \LE{} operators from $\mb{R}^{n^{\ell}}$ to $\mb{R}^{n^m}$. Each such map $T_{\gamma}$ can be represented by a matrix of size $n^{\ell}\times n^m$ which can further considered as a $(\ell+m)$-tensor $\bs{B}_{\gamma}$. %
\citet{maron2018invariant} showed that a specific basis for such operators can be characterized as follows: Each basis element will correspond to one of the $\bell{\ell+m}$ partitions in $\parspace_{\ell+m}$. In particular, given a partition $\gamma\in \parspace_{\ell+m}$, we have a corresponding basis \LE{} operator $T_{\gamma}$ and its tensor representation $\mathbf{B}_{\gamma}$ defined as follows:
\begin{equation}
\text{for any}~\bs{a} \in \mc{I}_{\ell +m}, ~~
\mathbf{B}_{\gamma}(\boldsymbol{a})= \begin{cases}1 & \boldsymbol{a} \in \gamma \\ 0 & \text { otherwise }\end{cases}
\end{equation}\label{eqn:kIGNbasis}
The collection $\mc{B} = \{ T_\gamma \mid \gamma \in \parspace_{\ell+m}\}$ form a basis for all $\LE_{\ell, m}$ maps. In \Cref{subsec:interpretation}, we will provide an interpretation of each element of $\mc{B}$, making it easy to reason its effect on an input tensor using a unified framework.

Before the main theorem, we also need to extend the \textnewnorm{} in \Cref{def:new-norm} from 2-tensor to high-order tensor. %
Intuitively, for $X \in \mb{R}^{n^{k}}$, $\newnorm{X}$ has $\bell{k}$ components, where each component corresponds to the normalized norm of $X_{\gamma}$, the slice of $X$ induced by $\gamma \in \parspace_{k}$. See \Cref{fig:slices} for examples of slices of a 3-tensor.  The \textnewnorm{} of input and output of a $\LE_{\ell, m}$ will be of dimension $\bell{\ell }$ and $\bell{m}$ respectively. See \Cref{subsec:extending-new-norm} for details.

\begin{figure}[htp]
  \centering
  \includegraphics[width=1\linewidth]{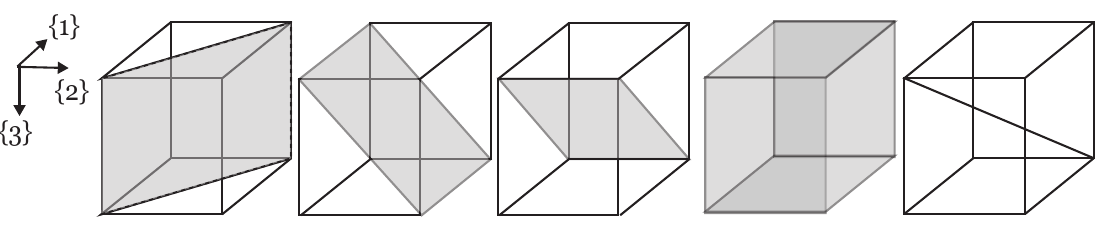}
\caption{Five possible ``slices'' of a 3-tensor, corresponding to $\bell{3}=5$ partitions of $[3]$. From left to right: a) $\{\{1, 2\}, \{3\}\}$ b) $\{\{1\}, \{2, 3\}\}$ c) $\{\{1, 3\}, \{2\}\}$ d) $\{\{1\}, \{2\}, \{3\}\}$ e) $\{\{1, 2, 3\}\}$.}
\label{fig:slices}
\end{figure}

 The following theorem characterizes the effect of each operator in $\mc{B}$ in terms of \textnewnorm{} of input and output, generalizing \Cref{prop:R2-R2} from matrix to high order tensor.

\begin{restatable}[Stability of \LE{} layers for $k$-IGN]{theorem}{kignlinearstability}
\label{thm:linear-layer-stability}
Let $T_{\gamma}: \mb{R}^{[0,1]^{\ell}} \rightarrow \mb{R}^{[0,1]^m}$ be a basis element of the space of $\LE_{\ell, m}$ maps where $\gamma \in \parspace_{\ell + m}$. %
If $\newnorm{X} \leqslant  \epsilon \one_{\bell{\ell}}$, then the \textnewnorm{} of $Y:=T_{\gamma}(X)$ satisfies  $\newnorm{Y}\leqslant  \epsilon \one_{\bell{m}}$ for all $\gamma \in \parspace_{\ell+m}$.
\end{restatable}

The proof relies on a new interpretation of elements of $\mc{B}$ in $k$-IGN. We give only an intuitive sketch using an example in the next subsection. See \Cref{subsec:linear-layer-stability-proof} for the proof.

\subsection{Interpretation of Basis Elements}
\label{subsec:interpretation}
For better understanding, we color the input axis %
 $\{1, ..., \ell\}$ as red and output axis $\{\ell +1, ..., \ell +m \}$ as blue. Each $T_{\gamma}$ corresponds to one partition $\gamma$ of $[\ell + m]$. %

For any partition $\gamma\in \parspace_{l+k}$, we can write this set as disjoint union $\gamma=S_1 \cup S_2 \cup S_3$ where $S_1$ is a set of set(s) of input axis, and $S_3$ is a set of set(s) of output axis. $S_2$ is a set of set(s) where each set contains both input and output axis. With slight abuse of notation, we omit the subscript $\gamma$ for $S_1, S_3, S_3$ when its choice is fixed or clear, and denote $\{\ell +1, ..., \ell +m \}$ as $\ell + [m]$.
As an example, one basis element of the space of $\LE_{3, 3}$ maps is $\gamma =\{ \{1,2\}, \{3,6\}, \{4\}, \{5\}\}$
\begin{equation}
\label{eq:partition-example}
\underbrace{S_1=\{\{\red{1},\red{2}\}\}}_{\text{Only has \red{input} axis}} \cup \underbrace{S_2=\{\{\red{3},\blue{6}\}\}}_{\substack{\text{has both} \\ \text{\red{input} and \blue{output} axis}}} \cup \underbrace{S_3=\{\{\blue{4}\}, \{\blue{5}\}\}}_{\text{only has \blue{output} axis}}
\end{equation}
where $1, 2, 3$ specifies the axis of input tensor and $4,5,6$ specifies the axis of the output tensor.
\begin{figure}[htp]
  \centering
  \vspace{-5pt}
  \includegraphics[width=.9\linewidth]{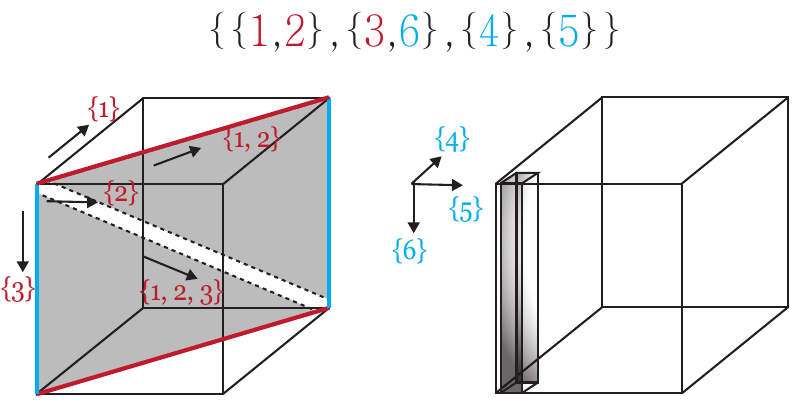}
\caption{An illustration of the one basis element of the space of $\LE_{3, 3}$. The partition is $\{\{1, 2\},\{3, 6\}, \{4\},\{5\}\}$. It selects area spanned by axis $\{1,2\}$ and $\{3\}$ (grey shaded), average over the (red) axis $\{1, 2\}$, and then align the resulting 1D tensor with axis $\{6\}$ in the output tensor, and finally replicate the slices along axis $\{4\}$ and $\{5\}$ to fill in the whole cube on the right. }
\label{fig:partition}
\end{figure}
Recall that there is a one-to-one correspondence between the partitions over $[\ell+m]$ and the base elements in $\mathcal{B}$ as in Eqn (\ref{eqn:kIGNbasis}). The basis element $T_\gamma$ corresponding to $\gamma = S_1 \cup S_2 \cup S_3$ operates on an input tensor $X\in \mb{R}^{n^\ell}$ and produce an output tensor $Y \in \mb{R}^{n^m}$ as follows:  \vspace{-5pt}
\begin{quotation}
Given input $X$, (step 1) obtain its slice $X_{\gamma}$ on $\Pi_1$ (selection axis),
(step 2) average $X_{\gamma}$ over $\Pi_2$ (reduction axis), resulting in $\xred$.
(step 3) Align $\xred$ on $\Pi_3$ (alignment axis) with $Y_{\gamma}$  and
(step 4) replicate $Y_{\gamma}$ along $\Pi_4$ (replication axis), resulting $\yrep$, a slice of $Y$. Entries of $Y$ outside $\yrep$ will be set to be 0.
\end{quotation}
In general, $\Pi_i$ can be read off from $S_1$-$S_3$. See \Cref{subsec:linear-layer-stability-proof} for details.
As a running example, \Cref{fig:partition} illustrates the basis element corresponding to $\gamma = S_1\cup S_2 \cup S_3$ where $S_1 =\{\{\red{1},\red{2}\}\} \cup S_2=\{\{\red{3},\blue{6}\}\} \cup S_3=\{\{\blue{4}\}, \{\blue{5}\}\}$. In the first step, given 3-tensor $X$, indexed by $\set{\set{1}, \set{2}, \set{3}}$ we select slices of interest $X_{\gamma}$ on $\Pi_1 = \set{\set{1,2}, \set{3}}$, colored in grey in the left cube of \Cref{fig:partition}. In the second step, we average $X_{\gamma}$ over axis $\Pi_2 = \set{\set{1,2}}$ to reduce 2-tensor $X_{\gamma}$, indexed by $\set{\set{1,2}, \set{3}}$ to a 1-tensor $\xred$, indexed by $\set{\set{3}}$. In the third step, the $\xred$ is aligned with $\Pi_3 = \set{\set{6}}$,  resulting in the grey cuboid $Y_{\gamma}$ indexed by $\set{\set{6}}$, shown in the right cube in \Cref{fig:partition}. Here the only difference between $\xred$ and $Y_{\gamma}$ is the index name of two tensors. In the fourth step, we replicate the grey cuboid $Y_{\gamma}$ over axis $\Pi_4 = \set{\set{4}, \set{5}}$ to fill in the cube, resulting in $\yrep$, indexed by $\set{\set{3}, \set{4}, \set{5}}$. %
Note in general $\yrep$ is a slice of $Y$ and does have to be the same as $Y$. %

These steps are defined formally in the Appendix.
For each of the four steps, we can control the \textnewnorm{} of output for each step (shown in \Cref{lemma:property-of-partition-norm} in Appendix), and therefore control the \textnewnorm{} of the final output for every basis element. See \Cref{subsec:linear-layer-stability-proof} for full proofs.

\section{Convergence of IGN in the Edge Weight Continuous Model}
\label{sec:convergence-ruiz}
\citet{ruiz2020graphon} consider  the convergence of $\|\Phi_c(W)-\Phi_c(W_n)\|_{L_2}$ in the \textit{graphon space}, where $W$ is the original graphon and $W_n$ is a piecewise constant graphon induced from graphs of size $n$ sampled from $W$ (to be defined soon). We call this model as the \textit{\edgeweight{}}. The main result of \citet{ruiz2021graph} is the convergence of \textit{continuous} \sGNN{} in the \textit{deterministic} sampling case where graphs are sampled from $W$ deterministically.
Leveraging our earlier stability result of linear layers of continuous IGNs in \Cref{thm:linear-layer-stability}, we can prove an analogous convergence result of cIGNs in the \edgeweight{} for both the deterministic and random sampling cases.

\textbf{Setup of the \edgeweight{}.} %
Given a graphon $W \in \mc{W}$ and a signal
$X \in \mb{R}^{[0, 1]\times d}$, the input of cIGN will be
$[W, \text{Diag}(X)] \in \mb{R}^{[0,1]^2 \times (1+d)}$. In the random sampling setting, we sample a graph of size $n$ from $W$ by setting the following edge weight matrix and discrete signal:
\begin{align}\label{eqn:det_gcn-random} \begin{split}
&[\widetilde{A_n}]_{ij} := W(u_i,u_j) \quad \mbox{and} \quad
[\widetilde{x_n}]_i := X(u_i)
\end{split}\end{align}
where $u_i$ is the $i$-th smallest point from $n$ i.i.d points sampled from uniform distribution on $[0,1]$.
We further lift the discrete graph $(\widetilde{A_n}, \widetilde{x_n})$ to a piecewise-constant graphon $\inducedW$ with signal $\inducedX$. Specifically, partition $[0,1]$ to be $I_1 \cup \ldots \cup I_n$ with $I_i = (u_i, u_{i+1}]$. We then define
\begin{align}\begin{split} \label{eqn:induced-cgcn-random}
&\inducedW(u,v) := {[\widetilde{A_n}]_{ij}} \times \mathrm{I}(u \in I_i)\mathrm{I}(v \in I_j) \quad \mbox{and} \\{}
&\inducedX(u) := [\widetilde{x_n}]_i \times \mathrm{I}(u \in I_i)
\end{split}\end{align}
 where $\mathrm{I}$ is the indicator function.
 Replacing the random sampling with fixed grid, i.e., let $u_i = \frac{i-1}{n}$,  we can get the deterministic \edgeweight{}, where $W_n$ and $X_n$ can be defined similarly as the lifting of a discrete sampled graph to a piecewise constant graphon. Note that $\inducedW$ is a piecewise constant graphon where each block is not of the same size, while all blocks $W_n$ are of size $\frac{1}{n} \times \frac{1}{n}$.
 We use $\widetilde{\cdot}$ to emphasize that $\widetilde{W_n}$/$\widetilde{X_n}$  are random variables, in contrast to the deterministic $W_n$/$X_n$.

We also need a few assumptions on the input and IGN.
\begin{assumption} \label{as:graphon-lip}
    The graphon $W$ is $A_1$-Lipschitz, i.e.
     $|W(u_2,v_2)-W(u_1,v_1)| \leqslant  A_1(|u_2-u_1|+|v_2-v_1|)$.
    \end{assumption}

    \begin{assumption} \label{as:filter-bound}
      The filter coefficients $c_{\gamma}$ are upper bounded by $\filterbound{}$. %
    \end{assumption}

    \begin{assumption} \label{as:signal-lip}
     The graphon signal $X$ is $A_3$-Lipschitz.
    \end{assumption}

    \begin{assumption} \label{as:activation-lip}
     The activation functions in IGNs are normalized Lipschitz, i.e.
     $|\rho(x)-\rho(y)| \leqslant  |x-y|$, and $\rho(0)=0$.
    \end{assumption}

Such four assumptions are quite natural and also adopted in \citet{ruiz2020graphon}. With AS 1-4, we have the following key proposition.
The proof leverages the stability of linear layers for $k$-IGN from \Cref{thm:linear-layer-stability}; see \Cref{app:proofs-EW} for details.

\begin{restatable}[Stability of $\Phi_c$]{proposition}{phistable}
\label{prop:Phi-stable}
If cIGN $\Phi_c: \mb{R}^{[0,1]^2 \times d_{\tn{in}}}   \rightarrow \mb{R}^{d_{\tn{out}}}$ satisfy AS\ref{as:filter-bound}, AS\ref{as:activation-lip} and $\newnorm{W_1-W_2}\leqslant  \epsilon \one_2$, then
 $\newnorm{\Phi_c(W_1) - \Phi_c(W_2)} = \| \Phi_c(W_1) - \Phi_c(W_2)\|_{L_2} \leqslant  C(A_2)\epsilon$ . The same statement still holds if we change the underlying norm of \textNewnorm{} from $L_2$ to $L_{\infty}$.
 \end{restatable}
 \begin{remark}
 Statements in \Cref{prop:Phi-stable} holds for discrete IGN $\Phi_d$ as well.
 \end{remark}

 From AS\ref{as:signal-lip} we can also bound the difference between the original signal $X$ and the induced signal ($X_n$ and $\inducedX$).

\begin{restatable}[]{lemma}{xdiffrandom}
\label{lem:x-diff}
Let $X \in \mb{R}^{[0,1]\times d}$ be an $A_3$-\lip{} graphon signal satisfying AS\ref{as:signal-lip}, and let $\inducedX$ and $X_n$ be the induced graphon signal as in \cref{eqn:det_gcn-random,eqn:induced-cgcn-random}. Then we have i) $\newnorm{X-X_n}$ converges to 0 and ii) $\newnorm{X-\inducedX}$ converges to 0 in probability.
\end{restatable}
We have the similar statements for $W$ as well. %
\begin{restatable}[]{lemma}{wdiffrandom}
\label{lem:w-diff}
If $W$ satisfies AS\ref{as:graphon-lip}, $\|W-W_n\|_{\tn{pn}} $ converges to 0. $\|W-\inducedW\|_{\tn{pn}} $ converges to 0 in probability.
\end{restatable}
The following main theorem (for $k$-cIGN of any order $k$) of this section can be shown by combining \Cref{prop:Phi-stable} with \Cref{lem:x-diff,lem:w-diff}; see  \Cref{app:proofs-EW} for details.

\begin{restatable}[Convergence of cIGN in the edge weight continuous model]{theorem}{EWconvergence}
\label{thm:EW-convergence}
Under the fixed sampling condition, IGN converges to cIGN, i.e., $\| \Phi_c\left([W, \dg{X}]\right) - \Phi_c([W_n, \dg{X_n}])\|_{L_2}$ converges to 0.

An analogous statement hold for the random sampling setting, where $\| \Phi_c([W, \dg{X}]) - \Phi_c([\inducedW, \dg{\inducedX}])\|_{L_2}$ converges to 0 in probability.

\end{restatable}

\section{Convergence of IGN in the Edge Probability Discrete Model}
\label{sec:EP-convergence}

In this section, we will consider the convergence setup of  \citet{keriven2020convergence}, which we call the \emph{\edgeprob}.
The major difference from the \edgeweight{} of \citet{ruiz2020graphon} is that (1) we only access 0-1 adjacency matrix instead of full edge weights and (2) the convergence error is measured in the graph space (instead of graphon space).

This model is more natural. However, we will first show a  negative result that in general IGN does not converge in the \edgeprob{} in \Cref{subsec:negative-result}. This motivates us to consider a relaxed setting where we estimate the edge probability from data. With this extra assumption, we can prove the convergence of \smallIGN{}, a subset of IGN, in the edge probability discrete model in  \Cref{subsec:smallign-convergnce}. Although this is not entirely satisfactory, we show that nevertheless, the family of functions that can be represented by \smallIGN{} is still rich enough to for example approximate any \sGNN{} arbitrarily well.

\subsection{Setup: Edge Probability Continuous Model}
\label{subsec: EP}
We first state the setup and results of \citet{keriven2020convergence}. We keep the notation close to the original paper for consistency.
A random graph model $(P, W, f)$ is represented as a probability distribution $P$ uniform over latent space $\mc{U}=[0,1]$, a symmetric kernel $W: \mc{U} \times \mc{U} \rightarrow [0, 1]$ and a bounded function (graph signal) $f: \mc{U} \rightarrow \mathbb{R}^{d_{z}}$. A random graph $G_n$ with $n$ nodes is then generated from $(P,W,f)$ according to latent variables $U := \set{u_1, ..., u_n}$ as follows:
\begin{gather*}
\forall j<i \leqslant n: \quad \mbox{graph node}~~ u_{i} \stackrel{i i d}{\sim} P, \quad z_{i}=f\left(u_{i}\right), \\
\quad \mbox{graph edge}~~ a_{i j} \sim \operatorname{Ber}\left(\alpha_{n} W(u_{i}, u_{j})\right) \numberthis
\label{equ:graph-generation-EP}
\end{gather*}
where Ber is the Bernoulli distribution and $\alpha_n$ controls the sparsity of sampled graph.  %
Note that in our case, we assume that the sparsification factor $\alpha_n = 1$ (which is the classical graphon model).
We define a degree function by $d_{W, P}(\cdot):= \int W(\cdot, u) d P(u)$.
We assume the following
\begin{gather*}
\|W(\cdot, u)\|_{L_\infty} \leqslant c_{\max }, \quad d_{W, P}(u) \geqslant c_{\min }, \\
\quad W(\cdot, u) \text { is }\left(c_{\text {Lip. }}, n_{\mathcal{U}}\right) \text {-piecewise Lipschitz. }\numberthis
\end{gather*}
A function $f: \mathcal{U} \rightarrow \mathbb{R}$ is said to be $\left(c_{\text {Lip. }}, n_{\mathcal{U}}\right) \text {-piecewise Lipschitz}$ if there is 
a partition $\mathcal{U}_1, ..., \mathcal{U}_n$ of $\mathcal{U}$ such that, for all $u, u'$ in the same $\mathcal{U}_i$, 
we have $|f(u)-f(u')|< c_{Lip.} d(u, u')$.
We introduce two normalized sampling operator $S_U$ and $S_n$ that sample a continuous function to a discrete one over $n$ points. %
For a function %
 $W': \mc{U}^{\otimes k} \rightarrow \mb{R}^{d_{\tn{out}}}$, $S_UW'(i_1, ..., i_k) := (\frac{1}{\sqrt{n}})^k(W'(u_{(i_1)}), ..., W'(u_{(i_k)})$ where $u_{(i)}$ is the i-th smallest number over $n$ uniform random samples over $[0,1]$ and $i_{1}, ..., i_{k}\in[n]$.
Similarly, $S_nW'(i_1, ..., i_k) := (\frac{1}{\sqrt{n}})^k \left(W'(\frac{i_1}{n}), ..., W'(\frac{i_k}{n})\right)$ %
Note that the normalizing constant will depend on the dimension of the support of $W'$.
We have $\|S_UW'\|_2 \leqslant  \| W'\|_{L_{\infty}}$ and $\|S_nW'\|_2 \leqslant  \| W'\|_{L_{\infty}}$. %

To measure the convergence error, we consider root mean square error at the node level: for a signal $x \in \mb{R}^{n^2 \times d_{\tn{out}}}$ and latent variables $U$, we define $\MSE_{U}(f, x)  \coloneqq \| S_U f - \frac{x}{n}\|_2 = (n^{-2} \sum_{i=1}^{n}\sum_{j=1}^{n}\left\|f\left(u_i, u_j\right)-x(i, j)\right\|^{2})^{1 / 2}.$
Again, there is a dependency on the input dimension -- the normalization term $n^{-2}$ will need to be adjusted when the input order is different from 2.
\subsection{Negative Result}%
\label{subsec:negative-result}

\begin{restatable}[]{theorem}{convfailure}
\label{thm:conv-failure}
Given any graphon $W$ with $c_{\text{max}} < 1$ and an IGN architecture (fix hyper-parameters like number of layers), there exists a set of parameters $\theta$ such that convergence of $IGN_{\theta}$ to c$IGN_{\theta}$ is not possible, i.e., $\MSE_U(\Phi_c\left([W, \dg{X}]\right), \Phi_d([A_n, \dg{\widetilde{x_n}}]))$ does not converge to 0 as $n\to \infty$, where $A_n$ is 0-1 matrix generated according to \cref{equ:graph-generation-EP}, i.e., $A_n[i][j] = a_{i,j}$. %
\end{restatable}
The proof of \Cref{thm:conv-failure} hinges on the fact that the input to IGN in discrete case is 0-1 matrix while the input to cIGN in the continuous case has edge weight upper bounded by $c_{\tn{max}} < 1$. The margin between 1 and $c_{\tn{max}}$ makes it easy to construct counterexamples.  See \Cref{app:missing-proofs-neg} for details.

Theorem \ref{thm:conv-failure} states that we cannot expect every IGN will converge to its continuous version cIGN.
As the proof of this theorem crucially uses the fact that we can only access 0-1 adjacency matrix, a natural question is what if we can estimate the edge probability from the data?
Interestingly, we can obtain the convergence of for a subset of IGNs (which is still rich enough), called \smallIGN{}, in this case. %

\subsection{Convergence of \smallIGN{}}
\label{subsec:smallign-convergnce}

Let $\widehat{W}_{n \times n}$ be the estimated $n\times n$ edge probability matrix from $A_n$. $\widetilde{W_n}$ is the induced graphon defined in \cref{eqn:induced-cgcn-random}.
 To analyze the convergence error for general IGN after edge probability estimation, we first decompose the convergence error of the interest using triangle inequality. Assuming the output is 1-tensor, then
\begin{align*}
\label{}
& \MSE_U(\Phi_c(W), \Phi_d(\widehat{W}_{n \times n})) \\
& = \|S_U \Phi_c(W)-\frac{1}{\sqrt{n}}\Phi_d (\widehat{W}_{n \times n}) \| \\
&\leqslant  \underbrace{\|S_U \Phi_c(W)- S_U\Phi_c(\inducedW)\|}_\text{First term: discretization error} +
\underbrace{\|S_U\Phi_c(\inducedW) - \Phi_dS_U(\inducedW)\|}_\text{Second term: sampling error} \\
& + \underbrace{\|\Phi_dS_U(\inducedW) - \frac{1}{\sqrt{n}}\Phi_d (\widehat{W}_{n \times n})\|}_\text{Third term: estimation error} \numberthis
\end{align*}
\begin{figure*}[htp!]
  \centering
  \includegraphics[width=.33\linewidth]{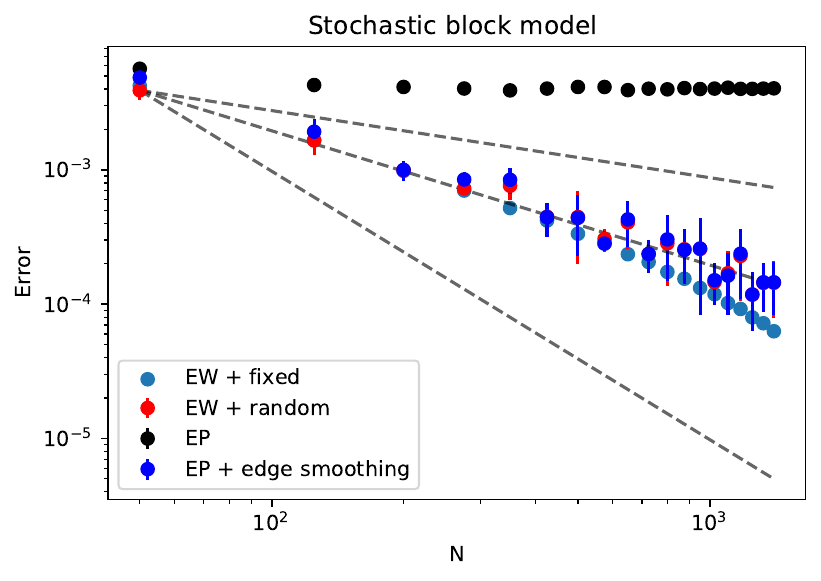}
  \includegraphics[width=.33\linewidth]{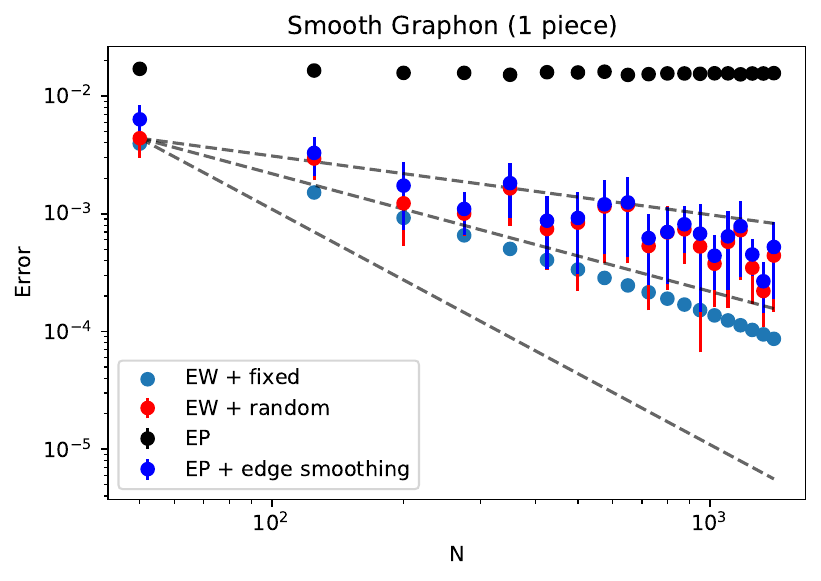}
  \includegraphics[width=.33\linewidth]{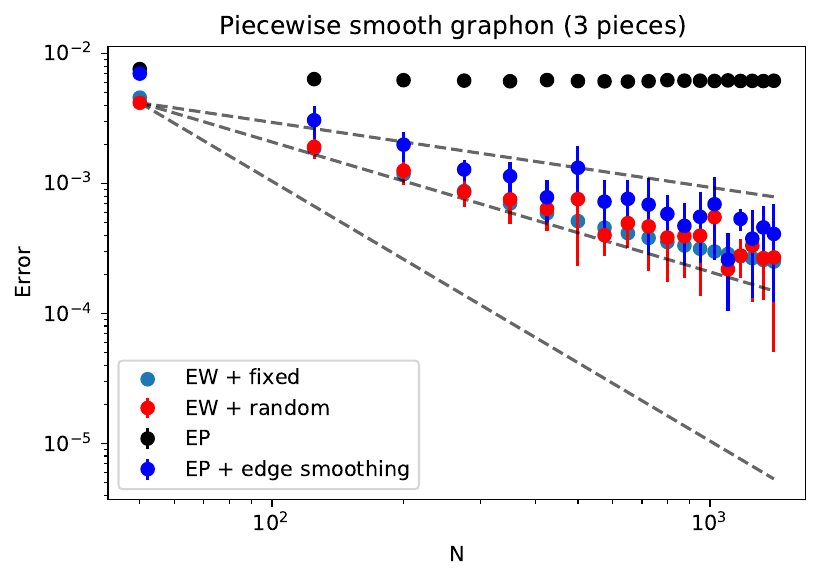}
  \label{fig:sfig1}
\vspace{-15pt}
\caption{The convergence error for three generative models: (left) stochastic block model, (middle) smooth graphon, (right) piece-wise smooth graphon. EW and EP stands for edge weight continuous model (\cref{eqn:det_gcn-random}) and edge probability discrete model (\cref{equ:graph-generation-EP}). Three dashed line in each figure indicates the decay rate of $n^{-0.5}, n^{-1}$ and $n^{-2}$. }
\vspace{-5pt}
\label{fig:convergence}
\end{figure*}
The three terms measure the different sources of error. First-term is concerned with the discretization error, which can be controlled via a property of $S_U$ and \Cref{prop:Phi-stable}. The  Second term concerns the sampling error from the randomness of $U$. This term will vanish if we consider only $S_n$ instead of $S_U$ under the extra condition stated below. The third term concerns the edge probability estimation error, which can also be controlled by leveraging existing literature on the statistical guarantee of the \textit{edge probability estimation} algorithm from \citet{zhang2015estimating}. \footnote{For better readability, here we only use the $W$ as input instead of $[W, \dg{X}]$. Adding $\dg{X}$ into the input is easy and is included in the full proof in \Cref{app:smallign-convergnce}. }

Controlling the second term is more involved. This is also the place where we have to add an extra assumption to constrain the IGN space in order to achieve convergence after edge smoothing.

\begin{definition}[\smallIGN{}]
Let $\inducedEW$ be a graphon with ``chessboard pattern'' \footnote{See full definition in \Cref{def:chessboard} in Appendix.}, i.e., it is a piecewise constant graphon where each block is of the same size. Similarly, define $\inducedEX$ as the 1D analog.
\smallIGN{} denotes a subset of IGN that satisfies $S_n\Phi_c([\inducedEW, \dg{\inducedEX}]) = \Phi_d S_n([\inducedEW, \dg{\inducedEX}])$.
\end{definition}

\begin{restatable}[convergence of \smallIGN{} in the edge probability discrete model]{theorem}{convergenceafterEM}
\label{thm:convergenceafterEM}
Assume AS 1-4, and let $\widehat{W}_{n \times n}$ be the estimated edge probability that satisfies $\frac{1}{n}\|W_{n \times n} - \widehat{W}_{n \times n}\|_2$ converges to 0 in probability. Let $\Phi_c, \Phi_d$ be continuous and discrete \smallIGN{}. Then  $\MSE_U\left(\Phi_c\left([W, \dg{X}]\right), \Phi_d\left([\widehat{W}_{n \times n}, \dg{\widetilde{x_n}}]\right)\right)$ converges to 0 in probability.
\end{restatable}

We leave the detailed proofs in \Cref{app:smallign-convergnce} with some discussion on the challenges for achieving full convergence results in the \Cref{remark:difficulty}.
We note that Theorem \ref{thm:convergenceafterEM} has a practical implication: It suggests that in practice, for a given unweighted graph (potentially sampled from some graphon), it may be beneficial to first perform edge probability estimation before feeding into the general IGN framework, to improve the architecture's stability and convergence.

Finally, although the convergence of \smallIGN{} is not entirely satisfactory, it contains some interesting class of functions that can approximate any \sGNN{} arbitrarily well. See \Cref{app:approx} for proof details.  %

\begin{restatable}{theorem}{smalligngcn}
\smallIGN{} can approximates \sGNN{} (both discrete and continuous ones) arbitrarily well on the compact domain in the $\|\cdot \|_{L_{\infty}}$ sense.
\end{restatable}

\section{Experiments}
\label{sec:exp}

We experiment 2-IGN on three graphon models of increasing complexity: Erdoes Renyi graph with $p=0.1$, stochastic block model of 2 blocks of equal size and probability matrix $[[0.1, 0.25], [0.25, 0.4]]$, a \lip{} graphon model with $W(u, v) = \frac{u + v + 1}{4}$, and a piecewise \lip{} graphon with $W(u, v) = \frac{u\%\frac{1}{3} + v\%\frac{1}{3} + 1}{4}$ where $\%$ is modulo operation. Similar to \cite{keriven2020convergence}, we consider untrained IGN with random weights to assess how convergence depends on the choice of architecture rather than learning.
 We use a 5-layer IGN with hidden dimension 16. We take graphs of different sizes as input and plot the error in terms of the norm of the output difference. The results are plotted in \Cref{fig:convergence}. %
See \Cref{app:all-exps} for full details and results. %

As suggested by the \Cref{thm:EW-convergence}, for both deterministic and random sampling, the error decreases as we increase the size of the sampled graph. %
Interestingly, if we take the 0-1 adjacency matrix as the input, the error does not decrease, which aligns with the negative result in \Cref{thm:conv-failure}. We further implement the edge smoothing algorithm \citep{eldridge2016graphons} and find that after the edge probability estimation, the error again decreases, as implied by \Cref{thm:convergenceafterEM}. We remark that although \Cref{thm:convergenceafterEM} works only for \smallIGN{}, our experiments for the general $2$-IGN with randomized initialized weights still show encouraging convergence results. Understanding the convergence of general IGN after edge smoothing is an important direction that we will leave for further investigation.

\section{Conclusion}
In this paper, we investigate the convergence property of a powerful GNN, Invariant Graph Network. We first prove a general stability result of linear layers in IGNs. We then prove a convergence result under the model of \citet{ruiz2020graphon} for both \IGN{} and high order \kIGN{}. %
Under the model of \citet{keriven2020convergence}
we first show a negative result that in general the convergence of every IGN is not possible.
Nevertheless, we pinpoint the major roadblock and prove that if we preprocess input graphs by edge smoothing \citep{zhang2015estimating}, the convergence of a subfamily of IGNs, called \smallIGN{}, can be obtained. As an attempt to quantify the size of \smallIGN{}, we also show that \smallIGN{} contains a rich class of functions that can approximate any \sGNN.

In the future, we would like to (1) further explore the expressive power of \smallIGN{} and
(2) investigate the convergence for the general IGNs under the \edgeprob{} model, or design variants with convergence property but are equally powerful.

\section*{Acknowledgement}
 The authors would like to thank anonymous reviewers for helpful comments. Chen Cai would like to thank Jinwoo Kim for helping out illustrations, Haggai Maron for helpful discussion, and Hy Truong Son for providing the Pytorch implementation of IGN. This work is in part supported by National Science Foundation under grants CCF-2112665 and IIS-2050360.
\newpage

\bibliography{main}
\bibliographystyle{icml2022}

\newpage
\onecolumn
\appendix

\section{Tables}
\label{app:tables}
We list the all \LE{} maps for $\mb{R}^{n} \rightarrow \mb{R}^{n\times n}$ and $\mb{R}^{n \times n} \rightarrow \mb{R}^{n}$ in \Cref{tab:R1-R2} and \Cref{tab:R2-R1} respectively.

We also summarize the notations used throughout the paper in \Cref{table:symbol_notation}.
\begin{table*}
\renewcommand{\arraystretch}{1.5}
\centering
\caption{Linear equivariant maps for $\mb{R}^{n} \rightarrow \mb{R}^{n\times n}$ and $\mb{R}^{[0,1]} \rightarrow \mb{R}^{[0,1]^2}$.}
\label{tab:R1-R2}
{\makegapedcells
\begin{tabular}[]{@{}llll@{}}
\toprule
Operations & Discrete & Continuous & Partitions\\ \midrule
1-3: Replicate to diagonal/rows/columns &
\makecell[l]{$T(A) = \text{Diag}(A)$\\  $T(A)_{i, j} = A_i$\\$T(A)_{i, j}= A_j$} &
\makecell[l]{$T(W)(u, v) = \mathrm{I}_{u=v}W(u)$\\ $T(W)(u,v) = W(u)$\\$T(W)(u,v)=W(v)$} &
\makecell[l]{ \{\{1,2,3\}\} \\ \{\{1,3\},\{2\}\} \\ \{\{1,2\},\{3\}\}}
\\

\hline

4-5: Replicate mean to diagonal/all matrix &
\makecell[l]{$T(A)_{i, i} = \frac{1}{n}A \one$\\$T(A)_{i, j} = \frac{1}{n}A \one $} &
\makecell[l]{$T(W)(u, v) = \mathrm{I}_{u=v}\int W(u) du$\\$T(W)(u,v) = \int W(u) du$} &
\makecell[l]{ \{\{1\},\{2,3\}\} \\ \{\{1\},\{2\},\{3\}\} }
\\
\bottomrule
\end{tabular}
}
\end{table*}

\begin{table*}
\renewcommand{\arraystretch}{1.5}
\centering
\caption{Linear equivariant maps for $\mb{R}^{n \times n} \rightarrow \mb{R}^{n}$ and $\mb{R}^{[0,1]^2} \rightarrow \mb{R}^{[0,1]}$.}
\label{tab:R2-R1}
{\makegapedcells
\begin{tabular}[]{@{}llll@{}}
\toprule
Operations & Discrete & Continuous & Partitions\\ \midrule
\makecell[l]{1-3: Replicate diagonal/row mean/\\columns mean} &
\makecell[l]{$T(A)=\text{Diag}^{*}(A)$\\$T(A)=\frac{1}{n}A\one$\\$T(A) = \frac{1}{n}A^T\one$} &
\makecell[l]{$T(W)(u) = W(u, u)$\\$T(W)(u) = \int W(u, v)dv$\\$T(W)(u) = \int W(u, v)du$} &
\makecell[l]{ \{\{1,2,3\}\} \\ \{\{1,2\},\{3\}\} \\ \{\{1,3\},\{2\}\} }
\\ \hline

\makecell[l]{4-5: Replicate mean of all elements/\\ mean of diagonal} &
\makecell[l]{$T(A)_i = \frac{1}{n^2}\one^T A \one $ \\$T(A)_i = \frac{1}{n}\one^T \text{Diag}(\dgdual{A}) \one$} &
\makecell[l]{$T(W)(u) = \int W(u, v) du dv$ \\ $T(W)(u) = \int \mathrm{I}_{u, v} W(u, v)dudv$} &
\makecell[l]{ \{\{1\},\{2\},\{3\}\} \\ \{\{1,2\},\{3\}\} }
\\
\bottomrule
\end{tabular}
}
\end{table*}

\begin{table}[!htbp]
\caption{Summary of important notations.}
\begin{center}
{
\begin{tabular}{@{}l|l@{}}
    \hline
    \toprule
    Symbol & Meaning \\
    \midrule
    \midrule
    $\one_n$ & all-one vector of size $n\times 1$ \\
    $\| \cdot \|_2/\| \cdot \|_{L_2}$ & 2-norm for matrix/
    graphon \\
    $\| \cdot \|_\infty / \| \cdot \|_{L_{\infty}}$ & infinity-norm for matrix/graphon \\
    $[\cdot, \cdot]$ & \makecell[l]{Given $A\in \mb{R}^{n^k \times d_1}, B\in \mb{R}^{n^k \times d_2}$, \\ $[A, B]$ is the concatenation of $A$ and $B$ along feature dimension. $[A, B] \in \mb{R}^{n^k \times (d_1 + d_2)}$.}
      \\
    $W: [0, 1]^2 \rightarrow [0,1]$ &  graphon \\
    $X\in \mb{R}^{[0,1]\times d}$ & 1D signal \\
    $\mc{W}$ & space of graphons \\
    $\newnorm{\cdot}$ & \textnewnorm{}. When the underlying norm is $L_{\infty}$ norm, we also use $\newnorminf{\cdot}$. \\
    $\mathrm{I}$ & indicator function \\
    $I$ & interval \\
    SGNN & spectral graph neural networks, defined in \Cref{eq:sgnn} \\
    $\LE_{\ell, m}$ & linear equivariant maps from $\ell$-tensor to $m$-tensor \\
    \hline
     & Notations related to sampling \\ \hline

    $W_n$ & Induced piecewise constant graphon from fixed grid \\
    $\widetilde{W_n}$ & Induced piecewise constant graphon from random grid \\
    $\inducedEW$ & \makecell[l]{Induced piecewise constant graphon from random grid, but resize the all \\ individual blocks to be of equal size (also called chessboard graphon in the paper). \\ $\inducedEW(I_i \times I_j) \coloneqq W(u_{(i)}, u_{(j)})$} \\
    $\sampleW$ & $n\times n$ matrix sampled from $W$; $\sampleW(i, j) = W(u_{i}, u_{j})$ \\
    $\widehat{W}_{n\times n} \in \mb{R}^{n \times n}$ & \makecell[l]{the estimated edge probability from graphs sampled according to \\ edge probability discrete model from \citet{zhang2015estimating}} \\
    $\widetilde{x_n} \in \mb{R}^{n \times d}$ & sampled signal $[\widetilde{x_n}]_i:=X(u_i)$ \\
    $X_n$ & induced 1D piecewise graphon signal from fixed grid \\
    $\inducedX$ & induced 1D piecewise graphon signal from random grid \\
    $S_U$ & normalized sampling operator for random grid. $S_Uf(i, j) = \frac{1}{n}(f(u_{(i)}), f(u_{(j)})$ \\
    $S_n$ & normalized sampling operator for fixed grid. $S_nf(i, j) = \frac{1}{n}(f(\frac{i}{n}), f(\frac{j}{n}))$\\
    $\MSE_U(x, f)$ & $\left(n^{-1} \sum_{i=1}^{n}\left\|x_{i}-f\left(u_{i}\right)\right\|^{2}\right)^{1 / 2}$ for 1D signal; $\left(n^{-2} \sum_{i}\sum_{j}\left\|x_{i, j}-f\left(u_{i}, u_{j}\right)\right\|^{2}\right)^{1 / 2}$ for 2D case\\
    $\alpha_n$ & a parameter that controls the sparsity of sample graphs. Set to be $1$ in the paper. \\
    \hline
     & Notations related to IGN \\ \hline
    $\bell{k}$ & Bell number: number of partitions of $[k]$. $\bell{2} = 2, \bell{3} = 5, \bell{4} = 15, \bell{5} = 52...$ \\
    $\parspace_k$ & space of all partitions of $[k]$ \\
    $\mc{I}_k$ & the space of indices. $\mc{I}_k:=\{(i_1, ..., i_k)| i_1\in [n], ..., i_k\in [n]\}$. Elements of $\mc{I}_k$ is denoted as $\bs{a}$ \\
    $\gamma \in [k]$ & \makecell[l]{partition of $[k]$. For example $\set{\set{1,2}, \set{3}}$ is a partition of $[3]$. \\
    The total number of partitions of $[k]$ is $\bell{k}$.} \\
    $\bs{a} \in \gamma$ & $\bs{a}$ satisfies the equivalence pattern of $\gamma$. For example, $(x, x, y) \in \set{\set{1, 2}, \set{3}}$ where $x, y, z\in [n]$.  \\
    $\gamma < \beta $ & given two partitions $\gamma, \beta \in \parspace_k$, $\gamma < \beta$ if $\gamma$ is finer than $\beta$. For example, $\set{1,2,3}<\set{\set{1,2}, \set{3}}$. \\
    $\bs{B}_{\gamma}$ & \makecell[l]{$l+m$ tensor; tensor representation of $\LE_{l, m}$ maps. \\ we differentiate $T_{\gamma}$ (operators) from  $\bs{B}_{\gamma}$ (tensor representation of operators)}\\
    $\mc{B}$ & a basis of the space of linear equivariant operations from $\ell$-tensor to $m$-tensor. $\mc{B}=\set{T_{\gamma}|\gamma \in \parspace_{l+k}}$ \\
    $T_c/T_d$ & linear equivariant layers for graphon (continuous) and graphs (discrete) \\
    $\Phi_c/\Phi_d$ & IGN for graphon (continuous) and graphs (discrete) \\
    $L^{(i)}$ & i-th linear equivariant layer of IGN \\
    $L$ & normalized graph Laplacian \\
    $T_{i}$ & basis element of the space of linear equivariant maps; sometimes also written as $T_{\gamma}$. \\

    \bottomrule
\end{tabular}
}
\end{center}
\label{table:symbol_notation}
\end{table}

\section{Missing Proofs from \Cref{sec:stability}}

\subsection{Extension of \textNewnorm{}}
\label{subsec:extending-new-norm}
There are three ways of extending  \textNewnorm{} 1) extend the definition of \textnewnorm{} to multiple channels 2) changing the underlying norm from $L_2$ norm to $L_{\infty}$ norm, and 3) extend \textNewnorm{} defined for 2-tensor to $k$-tensor.

First recall the definition \textnewnorm{}.
\partitionnormdef*

To extend \textnewnorm{} to signal $A \in \mb{R}^{n^2 \times d}$ of multiple channels, we denote $A = [A_{\cdot, 1} \in \mb{R}^{n^2 \times 1}, ..., A_{\cdot, d} \in \mb{R}^{n^2 \times 1}]$ where $[\cdot, \cdot]$ is the concatenation along channels. $\newnorm{A}:=\sum_{i=1}^d \newnorm{A_{\cdot, i}}$.
both for multi-channel signal both for graphs and graphons.

Another way of generalizing \textNewnorm{} is to change the $L_2$ to $L_{\infty}$ norm. We denote the resulting norm as $\newnorminf{\cdot}$. For $W \in \mc{W}, \newnorminf{W} := (\max_{u\in[0, 1]} W(u, u), \max_{u\in[0, 1], v\in[0, 1]} W(u, v))$. The discrete case and high order tensor case can be defined similarly as the $L_2$ case.

The last way of extending \textNewnorm{} to $k$-tensor $X \in \dtensor{k}{1}$ is to define the norm for each slice of $X$, i.e., $\newnorm{X}:=((\frac{1}{\sqrt{n}})^{|\gamma_1|}\|X_{\gamma_1} \|_2, ..., \frac{1}{\sqrt{n}})^{|\gamma_{\bell{k}}|}\|X_{\gamma_{\bell{k}}} \|_2)$ where $\gamma_{\cdot} \in \parspace_k$. Note how we order $(\gamma_1, ..., \gamma_{\bell{k}})$ can be arbitrary as long as the order is used consistent.

\subsection{Proof of stability of linear layer for 2-IGN}
\label{app:2-IGN-linear-stability}

\RtwoRtwo*
\begin{proof}

The statements hold in both discrete and continuous cases. Without loss of generality, we only prove the continuous case by going over all linear equivariant maps $\mb{R}^{[0, 1]^2} \rightarrow \mb{R}^{[0, 1]^2}$ in \Cref{tab:R2-R2}.
\begin{itemize}
\item 1-3: It is easy to see that the \textnewnorm{} does not increase for all three cases.

\item 4-6:
It is enough to prove case 4 only. Since $T(W)(*, u) = \int W(u, v)dv$, \dnorm{} $\|\dg{T(W)}\|_{L_2}^2 = \int (\int W(u, v)dv)^2 du \leqslant  \iint W^2(u, v)dudv $.  For matrix norm: $\|T(W)\|^2_{L_2} = \|\dg{T(W)}\|_{L_2} \leqslant   \iint W^2(u, v)dudv$. 
Therefore the statement holds for this linear equivariant operation.

\item 7-9: same as case 4-6.

\item 10-11: It is enough to prove the first case: average of all elements replicated on the whole matrix. The \dnorm{} is the same as the \mnorm{}. Both norms are decreasing so we are done.

\item 12-13: It is enough to prove only case 12. Since \dnorm{} is equal to \mnorm{}, and \dnorm{} is decreasing by Jensen's inequality we are done.

\item 14-15: Since \mnorm{} is the same as \dnorm{}, which stays the same so we are done.
\end{itemize}
As shown in all cases for any $W \in \mc{W}$ with $\|W\|_{\tn{pn}} < (\epsilon, \epsilon)$, $\|T_i(W)\|_{\tn{pn}} < (\epsilon, \epsilon)$. Therefore we finish the proof for $\mb{R}^{[0,1]^2} \rightarrow \mb{R}^{[0,1]^2}$.
We next go over all linear equivariant maps $\mb{R}^{[0, 1]} \rightarrow \mb{R}^{[0, 1]^2}$ in \Cref{tab:R1-R2} and prove it case by case.

\begin{itemize}
\item 1-3: It is enough to prove the second case. It is easy to see diagonal norm is preserved and $\|T(W)\|_2 = \|W\|_2 \leqslant  \epsilon$. Therefore $\|T(W)\|_{\tn{pn}} \leqslant  (\epsilon, \epsilon)$.

\item 4-5: It is enough to prove the second case. Norm on diagonal is no larger than $\|W\|$ by Jensen's inequality. The matrix norm is the same as the diagonal norm therefore also no large than $\epsilon$. Therefore $\|T(W)\|_{\tn{pn}} \leqslant  (\epsilon, \epsilon)$.
\end{itemize}

Last, we prove the cases for $\mb{R}^{[0, 1]^2} \rightarrow \mb{R}^{[0, 1]}$.

For cases 1-3, it is enough to prove case 2. Since the norm of the output is no large than the matrix norm of input by Jensen's inequality, we are done. Similar reasoning applies to cases 4-5 as well.
\end{proof}

\subsection{Proof of \Cref{thm:linear-layer-stability}}
\label{subsec:linear-layer-stability-proof}
We need a few definitions and lemmas first.
\begin{definition}[axis of a tensor]
Given a k-tensor $X \in \dtensor{k}{1}$ indexed by $(\name{1}, ..., \name{k})$. The axis of $X$, denoted as $\axis{X}$, is defined to be $\axis{X}:=(\name{1}, ..., \name{k})$. %
\end{definition}
As an example, the aixs of the first grey sub-tensor in \Cref{fig-app:slices}a, which is a $2$-tensor, is $\set{\set{1,2}, \set{3}}$.

\begin{definition}[replication of a tensor]
\label{def:replication}
Given a k-tensor $X \in \dtensor{k}{1}$ indexed by $(1, ..., k)$, replicating $X$ over new axis $(k+1, ..., k+d)$ means that the resulting new tensor $X'$ of $k+d$ dimension is $X'(i_1, ..., i_k, *, ..., *) := X(i_1, ..., i_k)$. %
\end{definition}

\begin{definition}[partial order of partitions]
\label{def:partial-order}
Given two partitions of $[k]$, denoted as $\gamma = \{\gamma_1, ..., \gamma_{d_1}\}$ and $\beta = \{\beta_1, ..., \beta_{d_2}\}$, we say $\gamma$ is finer than $\beta$, denoted as $\gamma < \beta$, if and only if 1) $\gamma \neq \beta$ and 2) for any $\beta_j \in \beta$, there exists $\gamma_i \in \gamma$ such that $ \beta_j \subseteq \gamma_i$.
\end{definition}
For example, $\{\{1,2,3\}\}$ is finer than $\{\{1, 2\}, \{3\}\}$ but $\set{\set{1,2}, \set{3}}$ is not comparable with $\set{\set{1,3}, \set{2}}$. Note that space of partitions forms a Hasse diagram under the partial order defined above (each set of elements has a least upper bound and a greatest lower bound, so that it forms a lattice). See \Cref{fig:hasse-diagram} for an example.

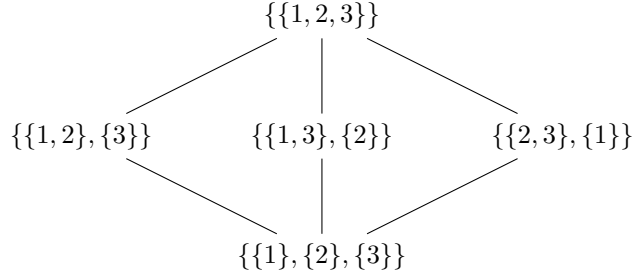
\begin{figure}%
\centering
\begin{tikzpicture}[scale=.8]
  \node (one) at (0,2) {$\set{\set{1,2,3}}$};
  \node (a) at (-4,0) {$\set{\set{1,2}, \set{3}}$};
  \node (b) at (0,0) {$\set{\set{1,3}, \set{2}}$};
  \node (c) at (4,0) {$\set{\set{2, 3}, \set{1}}$};
  \node (zero) at (0,-2) {$\set{\set{1}, \set{2}, \set{3}}$};
  \draw (zero) -- (a) -- (one) -- (b) -- (zero) -- (c) -- (one) ;
\end{tikzpicture}
\caption{Space of partitions forms a Hasse diagram under the partial order defined in \Cref{def:partial-order}.} Top to bottom corresponds to coarse partition to finer partition. %
\label{fig:hasse-diagram}
\end{figure}

\begin{definition}[average a $k$-tensor $X$ over $\Pi$]\label{def:averaging}
Let $X \in \dtensor{k}{1}$ be a $k$-tensor indexed by $\set{\set{1}, ..., \set{k}}$. Without loss of generality, let $\Pi = \set{\set{1}, ..., \set{d}}$. Denote the resulting $(k-d)$-tensor $X'$, indexed by $\set{\set{d+1}, ..., \set{k}}$.
By averaging $X$ over $\Pi$, we mean %
\begin{equation*}
 X'(\cdot):=\frac{1}{n^d}\sum_{t\in \mc{I}_d} X(t, \cdot).
 \end{equation*}
The definition can be extended to $\mb{R}^{[0, 1]^k}$ by replacing average with integral.
 \end{definition}

\begin{lemma}[properties of \textnewnorm{}]
\label{lemma:property-of-partition-norm}
We list some properties of the \textnewnorm{}. Although all lemmas are stated in the discrete case, the continuous version also holds. The statements also holds for $\newnorminf{\cdot}$ as well.
\begin{enumerate}[label=(\alph*)]
\item Let $X \in \dtensor{k}{1}$ be a $k$-tensor and denote one of its slices $X' \in \dtensor{k'}{1}$ with $k'\leqslant k$.  If $\newnorm{X} \leqslant \epsilon \one_{\bell{k}}$, then $\newnorm{X'} \leqslant \epsilon \one_{\bell{k'}}$.

\item Let $k'<k$. Let $X \in \dtensor{k}{1}$ be a $k$-tensor and $X' \in \dtensor{k'}{1}$ be the resulting $k'$-tensor after averaging over $k-k'$ axis of $X$. If $\newnorm{X} \leqslant \epsilon \one_{\bell{k}}$, then $\newnorm{X'} \leqslant \epsilon \one_{\bell{k'}}$.

\item Let $k'>k$. Let $X \in \dtensor{k}{1}$ be a $k$-tensor and $X'$ be the resulting $k'$-tensor after replicating $X$ over $k'-k$ axis of $X'$. If $\newnorm{X} \leqslant \epsilon \one_{\bell{k}}$, then $\newnorm{X'} \leqslant \epsilon \one_{\bell{k'}}$. 

\item Let $k'<k$ and $X \in \dtensor{k}{1}$ be a $k$-tensor such that it has only one non-zero slice $X_{\gamma}$ of order $k'$, i.e., 
if $\bs{a}\in \mc{I}_k, X(\bs{a})\neq 0 $, it implies $\bs{a} \in \gamma$.
If $\newnorm{X_{\gamma}} \leqslant \epsilon \one_{\bell{k'}}$, then $\newnorm{X} \leqslant \epsilon \one_{\bell{k}}$.
\end{enumerate}

\end{lemma}
\begin{proof}
We prove statements one by one. Note that although the proof is done for $L_2$ norm, we do not make use of any specific property of $L_2$ norm and the same proof can be applied to $L_{\infty}$ as well. Therefore all statements in the lemma apply to $\newnorminf{\cdot}$ as well.
\begin{enumerate}
    \item By the definition of \textnewnorm{} and slice in \Cref{def:slice}, we know that any slice of $X'$ is also a slice of $X$, therefore any component of $\newnorm{X'}$ will be upper bounded by $\epsilon$, which concludes the proof.

    \item Without loss of generality, we can assume that $k' = k-1$ as the general case can be handled by induction. Let the axis of $X$ that is averaged over is axis $\set{1}$. To bound $\newnorm{X'}$, we need to bound the normalized norm of any slice of $X'$. Let $X'_{\gamma'}$ be arbitrary slice of $X'$. Since $X'$ is obtained by averaging over axis 1 of $X$, we know that $X'_{\gamma'}$ is the obtained by averaging over axis of 1 of $X_{\gamma}$, a slice of $X$, where $\gamma := \gamma' \cup \{\set{1}\}$. Since $\newnorm{X} \leqslant \epsilon \one_{\bell{k}}$, we know that $(\frac{1}{\sqrt{n}})^{|\gamma|}\| X_{\gamma}\| \leqslant \epsilon$. By Jensen's inequality, we have $(\frac{1}{\sqrt{n}})^{|\gamma'|}\| X'_{\gamma'}\| \leqslant (\frac{1}{\sqrt{n}})^{|\gamma|}\| X_{\gamma}\|$, and therefore $(\frac{1}{\sqrt{n}})^{|\gamma'|}\| X'_{\gamma'}\| \leqslant \epsilon $. Since $(\frac{1}{\sqrt{n}})^{|\gamma'|}\| X'_{\gamma'}\| \leqslant \epsilon $ holds for arbitrary slice of $X'$, we conclude that $\newnorm{X'} \leqslant \epsilon \one_{\bell{k'}}$.

    The proof above only handles the case of $k'=k-1$. The general case where $k-k'>1$ can be handled by evoking the proof above multiple times for different reduction axis.

    \item Similar to the \Cref{lemma:property-of-partition-norm} (b), we can handle general case by performing induction. Therefore without loss of generality,  we assume $X$ is indexed by $(\set{1}, ..., \set{k} )$ and $X'$ is indexed by $(\set{1}, ..., \set{k+1})$. Just as the last case, without loss of generality we assume that $X'$ is obtained by replicating $X$ over $1$ new axis, denoted as $\set{k+1}$. In other words, $\axis{X'} = \axis{X} \cup \set{\set{k+1}}$.

    To control $\newnorm{X'}$, we need to bound $(\frac{1}{\sqrt{n}})^{|\gamma|}\| X'_{\gamma}\|$ where $\gamma \in \parspace_{k+1}$. Since $X'$ is obtained from $X$ by replicating it over $\set{k+1}$,  $(\frac{1}{\sqrt{n}})^{|\gamma|}\| X'_{\gamma}\| = (\frac{1}{\sqrt{n}})^{|\beta|}\| X_{\beta}\|$ where $\beta = \gamma |_{[k]}$. As $\newnorm{X} \leqslant \epsilon \one_{\bell{k}}$, it implies that $(\frac{1}{\sqrt{n}})^{|\gamma|}\| X'_{\gamma}\| \leqslant \epsilon $ holds for any $\gamma \in \parspace_{k'}$. Therefore we conclude that $\newnorm{X'} \leqslant \epsilon \one_{\bell{k'}}$.

    \item %
    To bound $\newnorm{X}$, we need to bound the normalized norm of any slice of $X$. Let $X_{\beta}$ be arbitrarily slice of $X$ where $\beta \in \parspace_k$.  Since $\gamma$ and $\beta$ are partitions of $[k]$, there exist partitions that are finer than both $\beta$ and $\gamma$, where the notion of finer between two partitions is defined in \Cref{def:partial-order}. Among all partitions that satisfy such conditions, denote the most coarse one as $\alpha \in \parspace_k$. This can be done because the $\parspace_k$ is finite.
    Note that $|\alpha| < |\beta|$ and $|\alpha| < |\gamma|$.

    Since $X_{\alpha}$ is a slice of $X_{\gamma}$ and $\newnorm{X_{\gamma}} \leqslant \epsilon \one_{\bell{k'}}$,
    $(\frac{1}{\sqrt{n}})^{|\alpha|}\|X_{\alpha}\| \leqslant \epsilon$ according to \Cref{lemma:property-of-partition-norm} (a).
    As $X_{\alpha}$ is the slice of $X_{\beta}$ (implies $\| X_{\alpha} \leqslant X_{\beta}\|$ )
    and $\alpha$ is the most coarse partition that is finer than $\beta$ and $\gamma$ (implies $\|X_{\alpha}\|\geqslant \| X_{\beta}\|$ we have $\| X_{\beta}\| = \| X_{\alpha}\|$.
    This implies $(\frac{1}{\sqrt{n}})^{|\beta|}\| X_{\beta}\| \leqslant (\frac{1}{\sqrt{n}})^{|\alpha|}\| X_{\alpha}\| \leqslant \epsilon$.
    
    As $(\frac{1}{\sqrt{n}})^{k'}\| X_{\beta}\| \leqslant \epsilon $ holds for arbitrary slice $\beta$ of $X$, we conclude that $\newnorm{X}\leqslant \epsilon \one_{\bell{k}}$.
\end{enumerate}

\end{proof}

Now we are ready to prove the main theorem.
\kignlinearstability*
\begin{proof}
Without loss of generality, we first consider discrete cases of mapping from $X \in \mb{R}^{n^{\ell}}$ to $Y \in \mb{R}^{n^m}$. In general, each element $T_{\gamma}$ of linear permutation equivariant basis can be identified with the following operation on input/output tensors.
\begin{quotation}
Given input $X$, (step 1) obtain its subtensor $X_{\gamma}$ on a certain $\Pi_1$ (selection axis),
(step 2) average $X_{\gamma}$ over $\Pi_2$ (reduction axis), resulting in $\xred$.
(step 3) Align $\xred$ on $\Pi_3$ (alignment axis) with $Y_{\gamma}$  and
(step 4) replicate $Y_{\gamma}$ along $\Pi_4$ (replication axis), resulting $\yrep$, a slice of $Y$. Entries of $Y$ outside $\yrep$ will be set to be 0. In general, $\Pi_i$ can be read off from $S_1$-$S_3$. %
\end{quotation}
$\Pi_1$-$\Pi_4$ corresponds to different axis of input/output tensor and can be read off from different parts of $S_{\gamma} = S_1 \cup S_2 \cup S_3$ as we introduced in the main text. Note such operation can be naturally extended to the continuous case, as done in \Cref{tab:R1-R2,tab:R2-R2,tab:R2-R1} for $2$-IGN. We next give detailed explanations of each step.

\begin{figure}[htp]
  \centering
  \includegraphics[width=.8\linewidth]{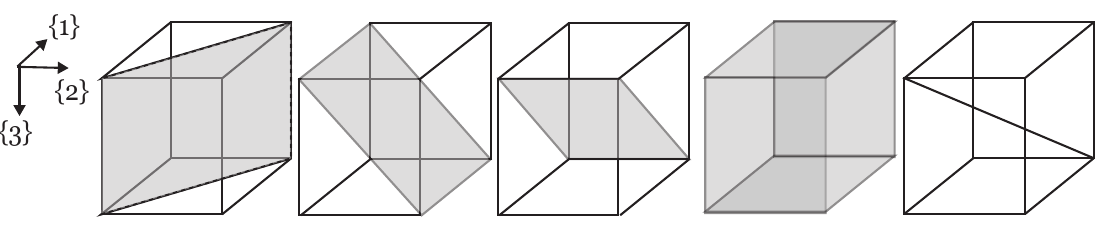}
\caption{Five ``slices'' of a 3-tensor, corresponding to $\bell{3}=5$ partitions of $[3]$. From left to right: a) $\{\{1, 2\}, \{3\}\}$ b) $\{\{1\}, \{2, 3\}\}$ c) $\{\{1, 3\}, \{2\}\}$ d) $\{\{1\}, \{2\}, \{3\}\}$ e) $\{\{1, 2, 3\}\}$.}
\label{fig-app:slices}
\end{figure}

\textbf{First step ($X \rightarrow X_\gamma$):
select $X_{\gamma}$ from $X$ via $\Pi_1$.}

$\Pi_1$ corresponds to $$S|_{[\ell]}\coloneqq\{s\cap [l]\mid s\in S \text{ and } s\cap [l] \neq \emptyset\}. $$ It specifies the what parts (such as diagonal part for 2-tensor) of the input $\ell$-tensor is under consideration. We denote the resulting subtensor as $X_{\gamma}$. See \Cref{def:slice} for formal definition.
As an example in \Cref{eq:partition-example}, $\Pi_1$ corresponds to $\{\{1, 2\},\{3\}\}$, meaning we select a 2-tensor with axises $\{1, 2\}$ and $\{3\}$. Note that the cardinality $|S|_{[\ell]}| = |(S_1 \cup S_2)|_{[\ell]}|\leqslant  l$ encodes the order of $X_{\gamma}$.

\textbf{Second step ($X_\gamma \rightarrow \xred$): average of $X_{\gamma}$ over $\Pi_2$.}
 $\Pi_2$ corresponds axes in  $S_1\subset S|_{[\ell]}$, which tells us along what axis to average over $X_{\gamma}$.
 It will reduce the tensor $X_{\gamma}$ of order $|S_1| + |S_2|$, indexed by $S|_{[\ell]} $, to a tensor of order $|S|_{[\ell]}|-|S_1|=|S_2|$, indexed by $S_2|_{[l]}$. Recall the definition of "averaging" in Definition \ref{def:averaging}.

 In the example of \Cref{fig-app:partition}, this corresponds to averaging over axis $\{\{1, 2\}\}$ %
 , reducing 2-tensor (indexed by axis $\{1,2\}$ and $\{3\}$) to 1-tensor (indexed by axis $\{3\}$). The normalization factor in the discrete case is $n^{|S_1|}$. We denote the tensor after reduction as $\xred$.

As the second step performs tensor order reduction, we end up with a tensor $\xred$ of order $|S_2|$.
The next two steps will describe how to fill in the output tensor $Y$ using $\xred$. To fill in $Y$, we will first align $\xred$ with $Y_{\gamma}$, a subtensor of $Y$, in the third step. We then replicate $Y_{\gamma}$ on $\Pi_4$ in the fourth step, resulting in $\yrep$, a sub-tensor of $Y$.  Finally, we fill all entries of $Y$ outside the subtensor $Y_\gamma$ to be zero.

\textbf{Third step ($\xred \rightarrow Y_{\gamma}$): align $\xred$ with $Y_{\gamma}$}.
To fill in $Y_{\gamma}$, we need to specify how the resulting $|S_2|$-tensor $\xred$ is \textit{aligned} with a certain $|S_2|$-subtensor $Y_\gamma$ of $Y$. After all, there are many ways of selecting a $|S_2|$-tensor from $Y$, which is indexed by $\{ \{l+1\}, ..., \{\ell +m\}\}$.
Specifically, set $Y_{\gamma}$ be the $|S_2|$-tensor indexed by $S_2 |_{\ell+[m]}$. We next define the precise relationship between $\xred$ and $Y_{\gamma}$. $\xred$ is indexed by $S_2|_{[l]}$ while $Y_{\gamma}$ is indexed by $S_2|_{l + [m]}$ and defined to be $Y_{\gamma}(\cdot) = \xred(\cdot)$. %
In the example of \Cref{fig-app:partition}, $\xred$ is a 1D tensor indexed by $\{3\}$ and $Y_{\gamma}$ (the grey cuboid on the right cube of \Cref{fig-app:partition}) is indexed by $\{6\}$.

\textbf{Fourth step ($Y_{\gamma} \rightarrow \yrep$): replicating $Y_{\gamma}$ over $\Pi_4$}. $\Pi_4$ corresponds to axes in $S_3$. It will be used to specify along what axis (axes) we will replicate the $|S_2|$-tensor $Y_{\gamma}$ over. Recall that $Y_{\gamma}$ is indexed by $S_2|_{l + [m]}$. Let $\yrep$ be a subtensor of $Y \in \mb{R}^{n^{l}}$ indexed by $(S_2 \cup S_3)|_{l + [m]}$. Obviously, the tensor $Y_{\gamma}$ output from the Third step is a subtensor of $\yrep$. 
Without loss of generality, let the first $|S_2|$ component are indexed by $S_2|_{l + [m]}$ and the rest components are indexed by $S_3|_{l + [m]}$. The mathematical definition of the fourth step is then $\yrep(\cdot, t) := Y_{\gamma}(\cdot)$ for all $t\in [n]^{|S_3|}$. Note that the order of $\yrep$ can be smaller than order of $Y$.

The example in \Cref{eq:partition-example} has $S_3=\{\{4\}, \{5\}\}$, which means that we will replicate the 1-tensor along axis $\{4\}$ and $\{5\}$. Note that in general, we do not have to fill in the whole $m$-tensor (think about copy row average to diagonal in \Cref{tab:R2-R2}).

\begin{figure}[htp]
  \centering
  \vspace{-5pt}
  \includegraphics[width=.4\linewidth]{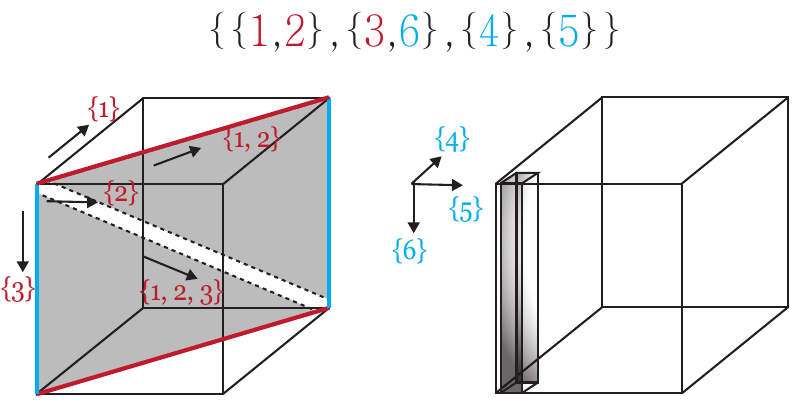}
\caption{An illustration of the one linear equivariant basis from $\mb{R}^{n^3} \rightarrow \mb{R}^{n^3}$. The partition is $\{\{1, 2\},\{3, 6\}, \{4\},\{5\}\}$. It selects area spanned by axis $\{1,2\}$ and $\{3\}$ (grey shaded), average over the (red) axis $\{1, 2\}$, and then align the resulting 1D slice with axis $\{6\}$ in the output tensor, and finally replicate the slices along axis $\{4\}$ and $\{5\}$ to fill in the whole cube on the right. }
\label{fig-app:partition}
\end{figure}

After the interpretation of general linear equivariant maps in $k$-IGN, We now show that %
if $\newnorm{X} \leqslant  \epsilon \one_{\bell{\ell}}$, then $T_{\gamma}(X)\leqslant  \epsilon \one_{\bell{m}}$ holds for all $\gamma$. This can be done easily with the use of \Cref{lemma:property-of-partition-norm}. %

For any partition of $[\ell +m]$ $\gamma$, according to the first step we are mainly concerned about the $\newnorm{X_{\gamma}}$ instead of $\newnorm{X}$.  Since $X_{\gamma}$ is a slice of $X$, then if $\newnorm{X} \leqslant \epsilon \one_{\bell{\ord{X}}}$, by \Cref{lemma:property-of-partition-norm} (a), then $\newnorm{X_{\gamma}} \leqslant \epsilon \one_{\bell{|S_1|+|S_2|}}$.

According to the second step and \Cref{lemma:property-of-partition-norm} (b), we can also conclude that $\newnorm{\xred} \leqslant \epsilon \one_{\bell{|S_2|}}$.

For the third step of align $\xred$ with $Y_{\gamma}$, it is quite obvious that $\newnorm{Y_{\gamma}} = \newnorm{\xred} \leqslant \epsilon \one_{\bell{|S_2|}}$.

For the fourth step of replicating $Y_{\gamma}$ over $\Pi_4$ to get $\yrep$, by \Cref{lemma:property-of-partition-norm} (c), we have $\newnorm{\yrep} \leqslant \epsilon \one_{\bell{|S_2|+|S_3|}}$.

Lastly, we evoke \Cref{lemma:property-of-partition-norm} (d) to get $\newnorm{Y} \leqslant \epsilon \one_{\bell{m}}$, which concludes our proof.

\end{proof}

\begin{remark}[On the difference from Incidence Networks for Geometric Deep Learning.]
\label{remark:difference}

A recent preprint Incidence Networks for Geometric Deep Learning \cite{albooyeh2019incidence} characterize the linear equivariant maps between incidence tensor, which encodes the combinatorial structure of graphs and its higher order analog simplicial complex and polytopes. \cite{albooyeh2019incidence} characterizes the linear permutation equivariant maps in terms of pooling and broadcasting operations. The pooling and broadcasting operations is the same as the averaging and replication operation defined in \Cref{def:averaging} and \Cref{def:replication}.

The main difference of \cite{albooyeh2019incidence} from our paper is 1) their motivation is to characterize the linear permutation equivariant maps between incidence tensors while in our paper, the similar characterization (in the case of linear permutation equivariant maps of $k$-IGN) serves as a building block for our convergence proof; 2) the characterization in \cite{albooyeh2019incidence} is slightly more general as incidence tensor can have different length for different axis while tensors considered in our case has the same length across all axis.
\end{remark}

\section{Missing Proofs from \Cref{sec:convergence-ruiz} (Edge Weight Continuous Model)}
\label{app:proofs-EW}
First we need a lemma on the distribution of gaps between $n$ uniform sampled points on $[0,1]$.
\begin{lemma}
\label{lemma:lengh-distribution}
Let $u_{(i)}$ be $n$ points uniformly sampled on $[0,1]$, sorted from small to large with $u_{(0)}=0$ and $u_{(n+1)}=1$. Let $D_i = u_{(i)} - u_{(i-1)}$. All $D_i$s have same distribution, which is $\tn{Beta}(1, n)$. 
In particular, expectation of $D_i$ $\mathbb{E}(D_i) = \frac{1}{n+1}$, $\mathbb{E}(D_i^2) = \frac{2}{(n+1)(n+2)}$, $\mathbb{E}(D_i^3) = \frac{6}{(n+1)(n+2)(n+3)}$. 
\end{lemma}
\begin{proof}
By a symmetry argument, it is easy to see that all intervals follow the same distribution.
For the first interval, the probability all the $n$ points are above $x$ is $(1-x)^n$ so the density of the length of the first (and so each) interval is $n(1-x)^{n-1}$. This is a Beta distribution with parameters $\alpha=1$ and $\beta=n$ The expectation of higher moments follows easily. Note that although the intervals are identically distributed, they are not independently distributed, since their sum is 1.
\end{proof}

\xdiffrandom*
\begin{proof}
We first bound the $\|X -X_n\|_{L_2[0, 1]}$ and $\|X -\inducedX\|_{L_2[0, 1]}$. %
For the first case, partitioning the unit interval as $I_i = [(i-1)/n, i/n]$ for $1 \leqslant  i \leqslant   n$ (the same partition used to obtain $x_n$, and thus $X_n$, from $X$), we can use the Lipschitz property of $X$ to derive
\begin{align*}
\left\|X-X_n\right\|^2_{L_2(I_i)} \leqslant  A_3^2\int_0^{1/n} u^2 du =  \frac{A_3^2}{3n^3} %
\end{align*}
We can then write
$\|X-X_n\|^2_{{L_2([0,1])}} = \sum_{i}\left\|X-X_n\right\|^2_{L_2(I_i)} \leqslant  \frac{A_3}{3n^2} $, %
which implies that $\|X-X_n\|_{{L_2([0,1])}} \leqslant  \sqrt{\frac{A_3}{3n^2}}$.

For the second case, since
$
\|X-\inducedX\|^2_{{L_2([0,1])}} = \sum_{i}\|X-\inducedX\|^2_{L_2(I_i)} %
$
, we will bound the $\left\|X-\inducedX\right\|^2_{L_2(I_i)}$.
As
\begin{equation*}
\left\|X-\inducedX\right\|^2_{L_2(I_i)} \leqslant  A_3^2\int_0^{D_i} u^2 du = A_3 D_i^{3} / 3
\end{equation*}
therefore
\begin{equation*}
\left\|X-\inducedX\right\|^2_{L_2(I)} = \sum_i \left\|X-\inducedX\right\|^2_{L_2(I_i)} \leqslant  A_3/3 \sum_i D_i^3
\end{equation*}

where $D_i$ stands for the length of $I_i$, which is a random variable due to the random sampling.

According to \Cref{lemma:lengh-distribution}, all $D_i$ are identically distributed and follows the Beta distribution $B(1, n-1)$. The expectation $E(D_i^3) = \frac{6}{n(n+1)(n+2)}$. %
Since by Jensen's inequality $E(\sqrt{Y}) \leqslant  \sqrt{E(Y)} $ holds for any positive random variable $Y$, $E(\sqrt{\frac{A_3}{3} \sum_i D_i^3)} \leqslant  \sqrt{E(\frac{A_3}{3} \sum_i D_i^3)} = \sqrt{\frac{A_3}{3} \frac{1}{n(n+2)}} = \Theta(\frac{1}{n})$. Using Markov inequality, we can then upper bound the
\begin{equation}
 P(\|X-\inducedX\|_{L_2(I)} \geq \epsilon) \leqslant
 P(\sqrt{\frac{A_3}{3} \sum_i D_i^3} \geq \epsilon) \leq
 \frac{E(\sqrt{\frac{A_3}{3} \sum_i D_i^3})}{\epsilon} =
 \Theta(\frac{1}{n\epsilon})
\end{equation}
Since the $P(\|X-\inducedX\|_{L_2(I)} \geq \epsilon)$ goes to 0 as $n$ increases, we conclude that $\newnorm{X-\inducedX}$ converges to 0 in probability. %
\end{proof}

\wdiffrandom*
\begin{proof}
For the first case, partitioning the unit interval as $I_i = [(i-1)/n, i/n]$ for $1 \leqslant  i \leqslant   n$, we can use the graphon's Lipschitz property to derive
\begin{align*}
\left\|W-W_n\right\|_{L_1(I_i\times I_j)} \leqslant  A_1\int_0^{1/n} \int_0^{1/n} |u| du dv &+ A_1\int_0^{1/n} \int_0^{1/n} |v| dv du =  \frac{A_1}{2n^3} + \frac{A_1}{2n^3} = \frac{A_1}{n^3} \text{.}
\end{align*}
We can then write
$\|W-W_n\|_{{L_1([0,1]^2)}} = \sum_{i,j}\left\|W-W_n\right\|_{L_1(I_i\times I_j)} \leqslant  n^2 \frac{A_1}{n^3} = \frac{A_1}{n}
$
which, since $W-W_n: [0,1]^2 \to [-1,1]$, implies
$\|W-W_n\|_{{L_2([0,1]^2)}} \leqslant  \sqrt{\|W-W_n\|_{{L_1([0,1]^2)}}} \leqslant  \wbound \text{.}$ The second last inequality holds because all entries of $W-W_n$ lies in $[-1, 1]$.

Similarly, $\|\text{Diag}(W-W_n)\|_{L_2[0,1]} \leqslant  \sqrt{\|\text{Diag}(W-W_n)\|_{L_1[0,1]}} \leqslant  \sqrt{2nA_1\int_0^{1/n}udu} = \sqrt{\frac{A_1}{n}}$. Therefore we conclude the first part of the proof.

For the second case, \dnorm{} is similar to the proof of \Cref{lem:x-diff} so we only focus on the $\|W-W_n\|_{{L_2([0,1]^2)}}$. Since $W-\inducedW: [0,1]^2 \to [-1,1]$ implies
\begin{equation*}
\|W-\inducedW\|_{{L_2([0,1]^2)}} \leqslant
\sqrt{\|W-\inducedW\|_{{L_1([0,1]^2)}}} =
\sqrt{\sum_{i,j} \|W-\inducedW\|_{{L_1(I_i \times I_j)}}}
\end{equation*}
where
\begin{equation*}
\|W-\inducedW\|_{{L_1(I_i\times I_j)}} \leqslant  A_1\int_{I_v} \int_{I_u} |u| du dv + A_1\int_{I_u} \int_{I_v} |v| dv du =
\frac{A_1}{2} (D_iD_j^2+D_jD_i^2)
\end{equation*}
Therefore
\begin{equation}
\|W-\inducedW\|_{{L_2([0,1]^2)}} \leqslant
\sqrt{\|W-\inducedW\|_{{L_1([0,1]^2)}}} =
\sqrt{\sum_{i, j} \frac{A_1}{2} (D_jD_i^2 + D_iD_j^2)} =
\sqrt{A_1\sum_i D_i^2}
\end{equation}
where we use the $\sum_i D_i =1$ for the last equality.
Since by Jensen's inequality $E(\sqrt{Y}) \leqslant  \sqrt{E(Y)} $ for any positive random variable $Y$, $E(\sqrt{\sum_i D_i^2}) \leqslant  \sqrt{E(\sum_i D_i^2)} = \Theta(\frac{1}{\sqrt{n}})$ since $E(D_i^2) = \Theta(\frac{1}{n^2})$ by \Cref{lemma:lengh-distribution}. By Markov inequality, we then bound
\begin{equation*}
P(\|W-\inducedW\|_{{L_2([0,1]^2)}}>\epsilon) \leq
P(\sqrt{\|W-\inducedW\|_{{L_1([0,1]^2)}}} > \epsilon) \leqslant
\frac{E(\sqrt{\sum_i D_i^2})}{\epsilon} \leq
\Theta(\frac{1}{\sqrt{n}\epsilon})
\end{equation*}
\end{proof}
Therefore, we conclude that both $\newnorm{W-W_n}$ and $\newnorm{W-\inducedW}$ converges to 0. %

\phistable*
\begin{proof}
Without loss of generality, it suffices to prove for $2$-IGN as $k$-IGN follows the same proof with the constant being slightly different.
Since we have proved stability of every linear layers of IGN in \Cref{thm:linear-layer-stability}, the general linear layer $T$ is just a linear combinations of individual linear basis, i.e. $T = \sum_{\gamma} c_{\gamma}T_{\gamma}$ where $c_i \leqslant  A_2$ for all $i$ according to AS\ref{as:filter-bound}. Without loss of generality, We can assume $T(X)$ is of order 2 and have
\begin{align*}
\|T(W_1)-T(W_2)\|_{\tn{pn}} & = \|\sum_i c_{\gamma}T_{\gamma}(W_1-W_2)\|_{\tn{pn}} \\
& \leqslant  \sum_i \|c_{\gamma}T_{\gamma}(W_1-W_2)\|_{\tn{pn}} \\
& \leqslant  (\sum |c_{\gamma}|\epsilon, \sum |c_{\gamma}|\epsilon) = (15\filterbound \epsilon, 15 \filterbound \epsilon)
\end{align*}

To extend the result to nonlinear layer, note that AS\ref{as:activation-lip} ensures the 2-norm shrinks after passing through nonlinear layers. Therefore $\newnorm{\sigma \circ T(X) - \sigma \circ T(Y)} \leqslant  \newnorm{T(X) - T(Y)} = \newnorm{T(X-Y)}\leqslant  15A_2\newnorm{X-Y}$.
Repeating such process across layers, we finish the proof of the $L_{2}$ case. %

The extension to $L_{\infty}$ is similar to the case of $L_2$ norm. The main modification is to change the definition of the \textnewnorm{} from $L_2$ norm on different slices (corresponding to different partitions of $[\ell]$ where $\ell$ is the order of input) to $L_{\infty}$ norm. The extension to the case where input and output tensor is of order $\ell$ and $m$ is also straightforward according to \Cref{thm:linear-layer-stability}.

\end{proof}

\EWconvergence*
\begin{proof}
By \Cref{prop:Phi-stable}, it suffices to prove that $\newnorm{[W, \dg{X}]) - [W_n, \dg{X_n}]}$ and $\newnorm{[W, \dg{X}]) - [\inducedW, \dg{\inducedX}]}$ goes to 0.

$\newnorm{[W, \dg{X}]) - [W_n, \dg{X_n}]}$ is upper bounded by $(\Theta(\frac{1}{n^{1.5}}), \Theta(\frac{1}{n^{1.5}}))$ according to \Cref{lem:w-diff,lem:x-diff}, which decrease to 0 as $n$ increases. Therefore we finish the proof of convergence for the deterministic case.

For the random sampling case, by \Cref{lem:w-diff,lem:x-diff}, we know that both $\|W-\inducedW\|_{{L_2([0,1]^2)}}$ and $\|X-\inducedX\|_{L_2(I)}$ goes to 0 as $n$ increases in probability at the rate of $\Theta(\frac{1}{n^{1.5}})$. Therefore we can also conclude that the convergence of IGN in probability according to \Cref{prop:Phi-stable}.
\end{proof}

\section{Missing Proof from \Cref{sec:EP-convergence} (Edge Probability Continuous Model)}
\subsection{Missing Proof for \Cref{subsec:negative-result}}
\label{app:missing-proofs-neg}

\convfailure*
\begin{proof}
Given a fixed IGN architecture $\Phi_c$ that maps input $\mathbb{R}^{n^2 \times d_1}$ to $ \mathbb{R}^{n^k \times d_2}$, it suffices to show the case of $k=1$ and $d_2 = 1$. 
Under the case of $k=1$ and $d_2 = 1$, it suffice to show that single layer IGN may not converge. Let $IGN =  \sigma \circ L^{(1)}$ have only one linear layer, and let the input to IGN be $A$ in the discrete case and $W$ in the continuous case. For simplicity, we assume that graphon $W$ is constant $p$ on $[0,1]^2$.
As $A$ consists of only 0 and 1 and all entries of $W$ is below $c_{\text{max}}$, we can set weights of IGN such that its first linear layer consists of only identity map and bias term. By choosing bias term to be any number between $[-1, -c_{\text{max}}]$, $L^{(1)}$ map any number no large than $c_\text{max}$ to negative and maps 1 to positive.

Therefore $L^{(1)}(W)=0$ and $L^{(1)}(A)$ is a positive number $c\in \mb{R}^{+}$ on entries $(i, j)$ where $A(i, j)=1$. Let $\sigma$ be ReLU and $L^{(2)}$ be average of all entries. We can see that c$IGN(W)=0$ for all $n$ while $IGN(A)$ converges to  $\sigma(c)p$ as $n$ increases.

As the construction above only relies on the fact that there is a separation between $c_{\text{max}}$ and $1$ (but not on size $n$), it can be extended to deeper IGNs 
, which means the gap between c$IGN(W)$ and $IGN(A)$ will not decrease as $n$ increases. In the general case of $W$ not being constant, the only difference is that $IGN(A)$ will converge to be $\sigma(c)p^*$ where $p^*$ is a different constant that depends on $W$.
Therefore we conclude the proof. 
\end{proof}
\begin{remark}
The reason that the same argument does not work for \sGNN{} is that \sGNN{} always maintains $Ax$ in the intermediate layer. In contrast, IGN keeps both $A$ and $\dg{x}$ in separate channels, which makes it easy to isolate them to construct counterexamples. 

\end{remark}

\subsection{Missing Proofs from \Cref{subsec:smallign-convergnce}}
\label{app:smallign-convergnce}
\textbf{Notation.} For any $P, Q \in \mb{R}^{n \times n}$, define $d_{2, \infty}$, the normalized $2, \infty$ matrix norm, by
$d_{2, \infty}(P, Q)=n^{-1 / 2}\|P-Q\|_{2, \infty}:=\max _{i} n^{-1 / 2}\left\|P_{i, \cdot}-Q_{i, \cdot} \right\|_{2}$ where $P_{i, \cdot}, Q_{i, \cdot}$ are $i$-th row of $P$ and $Q$, respectively. Note that $d_{2, \infty}(P, Q) \geq  \frac{1}{n}\| P - Q\|_2$. %

 Let $S_U$ be the sampling operator for $W$, i.e., $S_U(W) = \frac{1}{n}[W(U_i, U_j)]_{n\times n}$. Note that as $U$ is randomly sampled, $S_U$ is a random operator. Denote $\sampleE$ as sampling on a fixed equally spaced grid of size $n\times n$, i.e. $S_nW = \frac{1}{n}[W(\frac{i}{n}, \frac{j}{n})]_{n \times n}$. $S_n$ is a fixed operator when $n$ is fixed.

Let $\widehat{W}_{n \times n}$ be the estimated edge probability from graphs $A$ sampled from $W$. Let $\inducedW$ be the piece-wise constant graphon induced from sample $U$ as \cref{eqn:induced-cgcn-random}. Similarly, denote $\sampleW$ be the $n\times n$ matrix realized on sample $U$, i.e., $\sampleW[i, j] = W(u_i, u_j)$. It is easy to see that $S_U(W)  = \frac{1}{n}\sampleW$. Let $\inducedEW$ be the graphon induced by $\sampleW$ with $n\times n$ blocks of the same size. In particular, $\inducedEW(I_i \times I_j) \coloneqq W(u_{(i)}, u_{(j)})$ where $I_i = [\frac{i-1}{n}, \frac{i}{n}]$. $E$ in the subscript is the shorthand for the ``blocks of equal size''. Similarly we can also define the 1D analog of $\inducedW$ and $\inducedEW$, $\inducedX$ and $\inducedEX$.

\textbf{Proof strategy.} We first state five lemmas that will be used in the proof of \Cref{thm:convergenceafterEM}.
\Cref{lem:property-snsx} concerns the property of normalized sampling operator $S_U$ and $S_n$.
\Cref{lem:inducedW-converges,lem:inducedEW-converges} concern the convergence of $\Linf{\inducedW-W}$ and $\Linf{\inducedEW-W}$. \Cref{lem:property-of-T} characterize the effects of linear equivariant layers $T$ and IGN $\Phi$ on \Linfnorm of the input and output. \Cref{lem:sxsn-sampling} bounds the \Linfnorm of the difference of stochastic sampling operator $S_U$ and the deterministic sampling operator $S_n$. \Cref{thm:convergenceafterEM} is built on the results from five lemmas and the existing result on the theoretical guarantee of edge probability estimation from \citet{zhang2015estimating}.

The convergence some lemmas states is almost surely convergence.  Convergence almost surely implies convergence in probability, and in this paper, all theorems concern convergence in probability.  Note that proofs of \Cref{lem:property-snsx,lem:inducedW-converges,lem:inducedEW-converges,lem:sxsn-sampling} for the $W$ and $X$ are almost the same. Therefore without loss of generality, we mainly prove the case of $W$.

\begin{definition}[Chessboard pattern]
\label{def:chessboard}
Let $u_i = \frac{i-1}{n}$ for all $i\in [n]$. A graphon $W$ is defined to have chessboard pattern if and only if there exists a $n$ such that $W$ is a piecewise constant on $[u_i, u_{i+1}]\times [u_j, u_{j+1}]$ for all $i, j\in [n]$. Similarly, $f: [0, 1] \rightarrow \mb{R}$ has 1D chessboard pattern if there exists $n$ such that $f$ is a piecewise constant on $[u_i, u_{i+1}]$ for all $i\in [n]$.
\end{definition}
See \Cref{fig:chessboard} for examples and counterexamples.

\begin{figure*}[htp!]
  \centering
  \includegraphics[width=.7\linewidth]{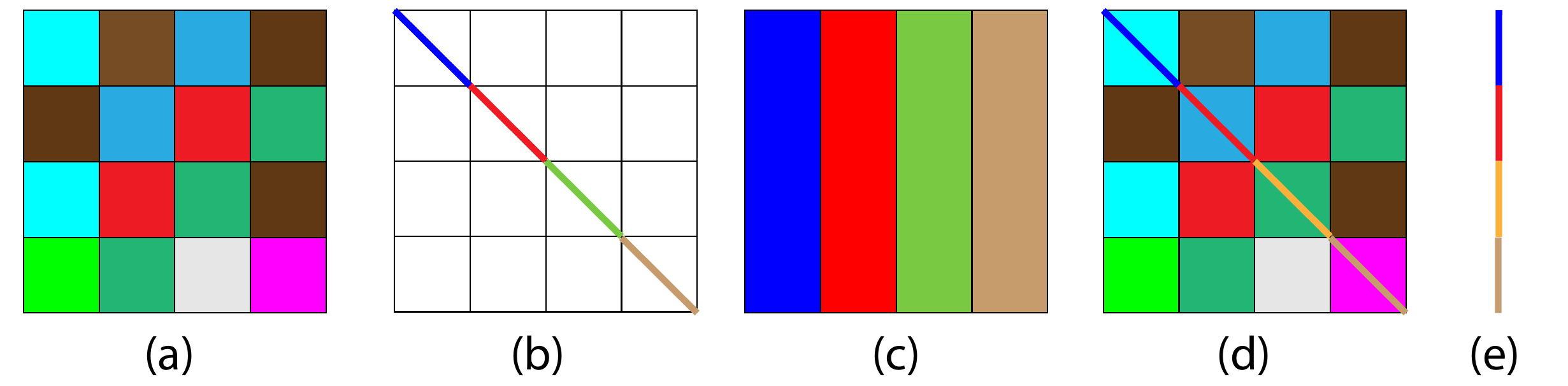}
\caption{(a) and (c) has chessboard pattern. (e) has 1D chessboard pattern. (d) does not has the chessboard pattern. (b) is of form $\dg{\inducedEf}$ and also does not have chessboard pattern, but in the case of IGN approximating Spectral GNN, (b) is represented in the form of c) via a linear equivariant layers of $2$-IGN.}
\label{fig:chessboard}
\end{figure*}

\begin{lemma}[Property of $S_n$ and $S_U$]\label{lem:property-snsx}
We list some properties of sampling operator $S_U$ and $S_n$
\begin{enumerate}
\item $S_U \circ \sigma = \sigma \circ S_U$. Similar result holds for $S_n$ as well.

\item $\|S_Uf_{\text{1d}}\| \leqslant  \| f_{\text{1d}} \|_{L_\infty}$ 

where $f_{\text{1d}}: [0,1] \rightarrow \mb{R}$. Similar result holds for $f_{\text{2d}}: [0,1]^2 \rightarrow \mb{R}$ and $S_n$ as well.

\end{enumerate}
\end{lemma}

\begin{lemma}\label{lem:inducedW-converges}
Let $W$ be $[0,1]^2\rightarrow \mb{R}$ and $X$ be $[0,1]\rightarrow \mb{R}$.
If $W$ is \lip{}, $\|\inducedW - W \|_{L_\infty}$ converges to 0 in probability. If $X$ is \lip{}, $\|\inducedX - X \|_{L_\infty}$ converges to 0 in probability.

\end{lemma}
\begin{proof}
Without loss of generality, we only prove the case for $W$. By the \lip{} condition of $W$, if suffices to bound the $Z_n\coloneqq\text{max}_{i=1}^{n}D_i$ where $D_i$ is the length of i-th interval $|u_{(i)} - u_{(i-1)}|$. Characterizing the distribution of the length of largest interval is a well studied problem \citep{renyi1953theory,pyke1965spacings,holst1980lengths}.  It can be shown that $Z_n$ follows  $P\left(Z_{n} \leqslant  x\right)=\sum_{j=0}^{n+1}\left(\begin{array}{c}n+1 \\ j\end{array}\right)(-1)^{j}(1-j x)_{+}^{n}$ with the expectation $E(Z_k) = \frac{1}{n+1}\sum_{i=1}^{n+1}\frac{1}{i} = \Theta(\frac{\log n}{n})$. By Markov inequality, we conclude that $\|\inducedW - W \|_{L_\infty}$ converges to 0 in probability. %

\end{proof}

\begin{lemma}\label{lem:inducedEW-converges}
Let $W$ be $[0,1]^2\rightarrow \mb{R}$ and $X$ be $[0,1]\rightarrow \mb{R}$.
If $W$ is \lip{}, $\| \inducedEW - W\|_{L_\infty}$ converges to 0 almost surely. If $X$ is \lip{}, $\| \inducedEX - X\|_{L_\infty}$ converges to 0 almost surely.
\end{lemma}

\begin{proof}
As $\inducedEW$ is a piecewise constant graphon and $W$ is \lip{}, we only need to examine $\text{max}_{i, j}\|(W - \inducedEW)(\frac{i}{n}, \frac{j}{n})\|$.

It is easy to see that $(W - \inducedEW)(\frac{i}{n}, \frac{j}{n}) = W(\frac{i}{n}, \frac{j}{n}) - W(u_{(i)}, u_{(j)})$ where $u_{(i)}$ stands for the i-th smallest random variable from uniform \iid samples from $[0,1]$. By the \lip{} condition of $W$, if suffices to bound $\|\frac{i}{n} - u_{(i)} \| + \|\frac{j}{n} - u_{(j)} \|$.
Glivenko-Cantelli theorem tells us that the $L_\infty$ of empirical distribution $F_n$ and  cumulative distribution function $F$ converges to 0 almost surely, i.e., $\text{sup}_{u\in [0,1]} |F(u) - F_n(u)| \rightarrow 0$ almost surely. Since $\text{max}_i \|u_{(i)} - \frac{i}{n} \| = \text{sup}_{u\in \{u_{(1)}, ..., u_{(n)}\}} |F(u) - F_n(u)|  \leqslant  \text{sup}_{u\in [0,1]} |F(u) - F_n(u)|$ when $F(u)=u$ (cdf of uniform distribution), we conclude that $\| \inducedEW - W\|_{L_\infty}$ converges to 0 almost surely.

\end{proof}

We also need a lemma on the property of the linear equivariant layers $T$.
\begin{lemma}[Property of $T_c$ and $\sigma$]
\label{lem:property-of-T}
Let $\sigma$ be nonlinear layer. Let $T_c$ be a linear combination of elements of basis of the space of linear equivariant layers of cIGN, with coefficients upper bounded. We have the following property about $T_c$ and $\sigma$
\begin{enumerate}
\item If $W$ is \lip{}, $T_c(W)$ is piecewise \lip{} on diagonal and off-diagonal. Same statement holds for $\Phi_c(W)$.

\item $S_n \circ \sigma (\inducedEW) = \sigma \circ S_n (\inducedEW)$.
\end{enumerate}
\end{lemma}

\begin{proof}
We prove two statements one by one.
\begin{enumerate}
\item
We examine the linear equivariant operators from $\mb{R}^{[0,1]^2}$ to $\mb{R}^{[0,1]^2}$ in \Cref{tab:R2-R2}. There are some operations such as ``average of rows replicated on diagonal'' will destroy the \lip{} condition of $T_c(W)$ but $T_c(W)$ will still be piecewise \lip{} on diagonal and off-diagonal.
Since $\sigma$ will preserve the \lip{}ness, $\Phi_c(W)$ is piecewise \lip{} on diagonal and off-diagonal.

\item This is easy to see as $\sigma$ acts on input pointwise.
\end{enumerate}
\end{proof}

\begin{lemma}\label{lem:sxsn-sampling}
Let $W$ be $[0,1]^2\rightarrow \mb{R}$
\begin{enumerate}
\item If $W$ is \lip{},  $\| S_UW - S_nW\|$ converges to 0 almost surely. Similarly, if $X$ is \lip{},  $\| S_U\dg{X} - S_n\dg{X}\|$ converges to 0 almost surely.

\item If $W$ is piecewise \lip{} on $S_1$ and $S_2$ where $S_1$ is the diagonal and $S_2$ is off-diagonal,  then $\| S_UW - S_nW\|$ converges to 0 almost surely. %

\end{enumerate}
\end{lemma}

\begin{proof}
Since the case of $X$ is essentially the same with that of $W$, we only prove the case of $W$.
\begin{enumerate}
\item As $n\| S_UW - S_nW\|_{\infty} \geq \| S_UW - S_nW\|$, it suffices to prove that $n\| S_UW - S_nW\|_{\infty} = \text{max}_{i, j}|W(u_{(i)}, u_{(j)}) - W(\frac{i}{n}, \frac{j}{n})|$  converges to 0 almost surely. Similar to \Cref{lem:inducedEW-converges}, using \lip{} condition of $W$ and Glivenko-Cantelli theorem concludes the proof.

\item This statement is stronger than the one above. The proof of the last item can be adapted here. As $W$ is $A_1$ \lip{} on off-diagonal region and $A_2$ \lip{} on diagonal,
\begin{align*}
 n\| S_UW - S_nW\|_{\infty} & =
  \text{max}_{i, j}\left|W(u_{(i)}, u_{(j)}) - W(\frac{i}{n}, \frac{j}{n})\right| \\
  = & \max\left(\text{max}_{i\neq j} \left|W(u_{(i)}, u_{(j)}) - W(\frac{i}{n}, \frac{j}{n})\right|, \text{max}_{i = j}\left|W(u_{(i)}, u_{(j)}) - W(\frac{i}{n}, \frac{j}{n})\right|\right).
\end{align*}
Using \lip{} condition on diagonal and off-diagonal part of $W$ and Glivenko-Cantelli theorem concludes the proof.
\end{enumerate}
\end{proof}
With all lemmas stated, we are ready to prove the main theorem.

\convergenceafterEM*
\begin{proof}

Using the triangle inequality %

\begin{align*}
& \MSE_U(\Phi_c \left([W, \dg{X}] \right), \Phi_d\left([\widehat{W}_{n \times n}, \dg{\widetilde{x_n}}]\right)) \\
& = \left\|S_U \Phi_c\left([W, \dg{X}]\right)-\frac{1}{\sqrt{n}}\Phi_d \left([\widehat{W}_{n \times n}, \dg{\widetilde{x_n}}]\right) \right\| \\
& = \|S_U \Phi_c\left([W, \dg{X}]\right)- S_U\Phi_c\left([\inducedW, \dg{\inducedX}]\right) + S_U\Phi_c\left([\inducedW, \dg{\inducedX}]\right) - \Phi_dS_U([\inducedW, \dg{\widetilde{x_n}}])  \\
&+ \Phi_dS_U([\inducedW, \dg{\inducedX}]) - \frac{1}{\sqrt{n}}\Phi_d ([\widehat{W}_{n \times n}, \dg{\inducedX}])\|\\
&\leqslant  \underbrace{\left\|S_U \Phi_c\left([W, \dg{X}]\right)- S_U\Phi_c\left([\inducedW, \dg{\inducedX}]\right)\right\|}_\text{First term: discretization error} +
\underbrace{\left\|S_U\Phi_c\left([\inducedW, \dg{\inducedX}]\right) - \Phi_dS_U([\inducedW, \dg{\inducedX}])\right\|}_\text{Second term: sampling error} \\
& + \underbrace{\left\|\Phi_dS_U([\inducedW, \dg{\inducedX}]) - \frac{1}{\sqrt{n}}\Phi_d \left([\widehat{W}_{n \times n}, \dg{\widetilde{x_n}}]\right)\right\|}_\text{Third term: estimation error} \numberthis \label{equ:threeterms} %
\end{align*}
The three terms measure the different sources of error. The first term is concerned with the discretization error. The second term concerns the sampling error from the randomness of $U$. This term will vanish if we consider only $S_n$ instead of $S_U$ for \smallIGN{}. The third term concerns the edge probability estimation error.

For the first term, it is similar to the sketch in \Cref{subsec:smallign-convergnce}. $\|S_U \Phi_c([W, \dg{X}]) - S_U \Phi_c ([\inducedW, \dg{\inducedX}]) \| = \| S_U(\Phi_c([W, \dg{X}]) - \Phi_c([\inducedW, \dg{\inducedX}]))\|$, if suffices to upper bound $\|\Phi_c([W, \dg{X}])-\Phi_c([\inducedW, \dg{\inducedX}])\|_{L_\infty}$ according to property of $S_U$ in \Cref{lem:property-snsx}. Since $\| \Phi_c ([W, \dg{X}]) - \Phi_c ([\inducedW, \dg{\inducedX}])\|_{L_\infty} \leqslant  C (\|W - \inducedW\|_{L_\infty} + \|\dg{X} - \dg{\inducedX}\|_{L_\infty})$ by \Cref{prop:Phi-stable}, and $\|W - \inducedW\|_{L_\infty}$ converges to 0 in probability according to \Cref{lem:inducedW-converges}, we conclude that the first term will converges to 0 in probability.

For the third term $\| \Phi_d S_U ([\inducedW, \dg{\inducedX}]) - \frac{1}{\sqrt{n}}\Phi_d ([\widehat{W}_{n \times n}, \dg{\widetilde{x_n}}])\| $$=\| \frac{1}{\sqrt{n}}( \Phi_d ([\sampleW, \dg{\widetilde{x_n}}]) - \Phi_d ([\widehat{W}_{n \times n}, \dg{\widetilde{x_n}}]) )\| $$=\newnorm{\Phi_d([\sampleW, \dg{\widetilde{x_n}}]) - \Phi_d([\widehat{W}_{n \times n}, \dg{\widetilde{x_n}}])}$,

it suffices to control the $\newnorm{[W_{n \times n}, \dg{\widetilde{x_n}}]-[\widehat{W}_{n \times n}, \dg{\widetilde{x_n}}]} = \frac{1}{n}\| \sampleW-\widehat{W}_{n\times n} \|_2 \leqslant  \|\sampleW-\widehat{W}_{n\times n}\|_{2, \infty}$, which will also goes to 0 in probability as $n$ increases according to the statistical guarantee of edge probability estimation of neighborhood smoothing algorithm \citep{zhang2015estimating}, stated in \Cref{thm:graphon-estimation} in \Cref{app:third-party}. 
Therefore by \Cref{prop:Phi-stable}, the third term also goes to 0 in probability.

Therefore the rest work is to control the second term $\|S_U \Phi_c\left([\inducedW, \dg{\inducedX}]\right) - \Phi_d S_U\left([\inducedW, \dg{\inducedX}]\right)\|$. Again, we use the triangle inequality
\begin{align*}
& \tn{Second term} \\
& = \left\|S_U \Phi_c\left([\inducedW, \dg{\inducedX}]\right) - \Phi_d S_U\left([\inducedW, \dg{\inducedX}]\right)\right\| \\
& \leqslant  \left\|S_U \Phi_c\left([\inducedW, \dg{\inducedX}]\right) - \sampleE \Phi_c\left([\inducedEW, \dg{\inducedEX}]\right) \right\| + \left\| \sampleE \Phi_c\left([\inducedEW, \dg{\inducedEX}]\right) - \Phi_d S_U\left([\inducedW, \dg{\inducedX}]\right)\right\| \\
& =  \left\|S_U \Phi_c\left([\inducedW, \dg{\inducedX}]\right) - \sampleE \Phi_c\left([\inducedEW, \dg{\inducedEX}]\right) \right\| + \left\| \sampleE \Phi_c\left([\inducedEW, \dg{\inducedEX}]\right) - \Phi_d \sampleE([\inducedEW, \dg{\inducedEX}])\right\|\\
& = \left\|S_U \Phi_c\left([\inducedW, \dg{\inducedX}]\right) - \sampleE \Phi_c\left([\inducedEW, \dg{\inducedEX}]\right) \right\|\\
& \leqslant  \left\|S_U \Phi_c\left([\inducedW, \dg{\inducedX}]\right) - S_U \Phi_c\left([\inducedEW, \dg{\inducedEX}]\right)\right\| + \left\|S_U \Phi_c\left([\inducedEW, \dg{\inducedEX}]\right) - \sampleE \Phi_c\left([\inducedEW, \dg{\inducedEX}]\right)  \right\| \\
& = \underbrace{\left\|S_U \left(\Phi_c ([\inducedW, \dg{\inducedX}]\right) - \Phi_c \left([\inducedEW, \dg{\inducedEX}]\right)\right\|}_\text{term $a$} +
\underbrace{\left\|(S_U - \sampleE)\Phi_c\left([\inducedEW, \dg{\inducedEX}]\right) \right\|}_\text{term $b$}
\end{align*}

The second equality holds because $S_U([\inducedW, \dg{\inducedX}]) = S_n([\inducedEW, \inducedEX])$ by definition of $\inducedEW$ and \smallIGN{} (See \Cref{remark:difficulty} for more discussion).  The third equality holds by the definition of \smallIGN{}.
We will bound the term a) $\|S_U (\Phi_c ([\inducedW, \dg{\inducedX}]) - \Phi_c ([\inducedEW, \inducedEX]))\|$ and b) $\|(S_U - \sampleE)\Phi_c([\inducedEW, \dg{\inducedEX}]) \|$ next.

For term a) $\|S_U (\Phi_c ([\inducedW, \dg{\inducedX}]) - \Phi_c ([\inducedEW, \inducedEX]))\|$, if suffices to prove that $\|\Phi_c ([\inducedW, \dg{\inducedX}]) -\Phi_c ([\inducedEW, \inducedEX]))\|_{L_\infty}$ converges to 0 in probability.
According to \Cref{prop:Phi-stable}, it suffices to bound the $\| [\inducedW, \inducedX] - [\inducedEW, \inducedEX] \|_{L_\infty}$. Because $[\inducedW, \inducedX] - [\inducedEW, \inducedEX] \|_{L_\infty} =  \|\inducedW - \inducedEW \|_{L_\infty} + \|\dg{\inducedX} - \dg{\inducedEX} \|_{L_\infty}) \leqslant  \|\inducedW - W \|_{L_\infty} + \|\inducedEW - W \|_{L_\infty} + \|\dg{\inducedX} - \dg{X} \|_{L_\infty} + \|\dg{\inducedEX} - \dg{X} \|_{L_\infty}$, we only need to upper bound $\|\inducedW - W \|_{L_\infty}$, $\|\inducedEW - W \|_{L_\infty}$, $\|\dg{\inducedX} - \dg{X} \|_{L_\infty})$ and $\|\dg{\inducedEX} - \dg{X} \|_{L_\infty})$, which are proved by \Cref{lem:inducedW-converges} and \Cref{lem:inducedEW-converges} respectively.

For term b) $\|(S_U - \sampleE)\Phi_c\left([\inducedEW, \dg{\inducedEX}]\right) \|$
\begin{align*}
&\left\|(S_U - \sampleE)\Phi_c\left([\inducedEW, \dg{\inducedEX}]\right) \right\| \\
&= \left\|(S_U \Phi_c\left([\inducedEW, \dg{\inducedEX}]\right) - \sampleE \Phi_c\left([\inducedEW, \dg{\inducedEX}]\right)\right\|\\
& \leqslant  \left\|(S_U \Phi_c\left([\inducedEW, \dg{\inducedEX}]\right) - S_U \Phi_c\left([W, \dg{X}]\right) \right\| + \left\|S_U \Phi_c\left([W, \dg{X}]\right) - S_n \Phi_c\left([W, \dg{X}]\right)\right\|  \\
& + \left\|S_n \Phi_c\left([W, \dg{X}]\right) - \sampleE \Phi_c\left([\inducedEW, \dg{\inducedEX}]\right)\right\|\\
& = \left\|(S_U (\Phi_c\left([\inducedEW, \dg{\inducedEX}]\right)-\Phi_c ([W, \dg{X}])) \right\| + \left\|S_U \Phi_c\left([W, \dg{X}]\right) - S_n \Phi_c\left([W, \dg{X}]\right)\right\| \\
& + \left\|\sampleE (\Phi_c\left([\inducedEW, \dg{\inducedEX}]\right)-\Phi_c ([W, \dg{X}]))\right\|
\end{align*}
For the first and last term, by the property of $S_U, S_n$ and $\Phi_c$, it suffices to bound $\|W-\inducedEW \|_{L_\infty}$ and $\|\dg{X}-\dg{\inducedEX} \|_{L_\infty}$. Without loss of generality, We only prove the case for $W$. As $\|W-\inducedEW \|_{L_\infty}$ converges to 0 almost surely by \Cref{lem:inducedEW-converges}, we conclude that the first and last term converges to 0 almost surely (therefore in probability).
For the second term $\|S_U \Phi_c\left([W, \dg{X}]\right) - S_n \Phi_c\left([W, \dg{X}]\right)\|$, $\Phi_c\left([W, \dg{X}]\right)$ is piecewise \lip{} on diagonal and off-diagonal according to \Cref{lem:property-of-T} 

, and it converges to 0 almost surely according to the second part of \Cref{lem:sxsn-sampling}.

As all terms converge to 0 in the probability or almost surely, we conclude that $\|S_U \Phi_c\left([W, \dg{X}]\right)-\Phi_d ([\widehat{W}_{n \times n}, \dg{\inducedX}]) \|$ converges to 0 in probability.
\end{proof}

\begin{remark}
\label{remark:difficulty}
Note that we can not prove $S_n \cdot \Phi_c (\inducedEW) = \Phi_d \cdot S_n (\inducedEW)$ in general.
The difficulty is that starting with $\inducedEW$ of chessboard pattern, after the first layer, pattern like \Cref{fig:chessboard}(e) may appear in $\sigma \circ T_1(\inducedW)$. If $T_2$ is just a average/integral to map $\dtensor{2}{1}$ to $\mb{R}$, then $S_n \circ T_2 \circ \sigma \circ T_1(\inducedW) = T_2 \circ \sigma \circ T_1(\inducedW)$ will not be equal to $T_2 \circ \sigma \circ T_1(S_n \inducedW)$. The reason is that both $\sigma \circ T_1(\inducedW)$ and $\sigma \circ T_1(S_n \inducedW)$ will no longer be of chessboard pattern (\Cref{fig:chessboard}(e) may occur). The diagonal in the $\sigma \circ T_1(\inducedW)$ has no effect after taking integral in $T_2$ as it is of measure 0. On the other hand, the diagonal in the matrix $\sigma \circ T_1(S_n \inducedW)$ will affect the average. Therefore in general,  $S_n \Phi_c (\inducedEW) = \Phi_d S_n (\inducedEW)$ does not hold.

\end{remark}

\section{\smallIGN{} Can Approximate Spectral GNN}
\label{app:approx}
\textbf{Definition of Spectral GNN.}
The \sGNN{} (SGNN) here stands for GNN with multiple layers of the following form $\forall j=1, \ldots d_{\ell+1}$,
\begin{align}
\label{eq:sgnn}
{\quad }z_{j}^{(\ell+1)}=\sigma \left(\sum_{i=1}^{d_{\ell}} h_{i j}^{(\ell)}(L) z_{i}^{(\ell)}+b_{j}^{(\ell)} 1_{n}\right) \in \mb{R}^{n}
\end{align}
where $L = D(A)^{-\frac{1}{2}} A D(A)^{-\frac{1}{2}}$ stands for normalized adjacency,\footnote{We follow the same notation as \citet{keriven2020convergence}, which is different from the conventional notation.} $z_j^{\ell}, b_j^{\ell} \in \mb{R}$ denotes the embedding and bias at layer $\ell$. $d_\ell$ stands for the number of output channels in $\ell$-th layer. 
 $h: \mb{R} \rightarrow \mb{R}, h(\lambda) = \sum_{k\geq 0} \beta_k \lambda^k, h(L)=\sum_{k} \beta_{k} L^{k}$, i.e., we apply $h$ to the eigenvalues of $L$ when it is diagonalizable. Extending $h$ to multiple input output channels which are indexed in $i$ and $j$, we have $h_{i j}^{(\ell)}(\lambda)=\sum_{k} \beta_{i j k}^{(\ell)} \lambda^{k}$. By defining all components of \sGNN{} for graphon, the continuous version of \sGNN{} can also be defined. %
See \citet{keriven2020convergence} for details.

We first prove IGN can approximate \sGNN{} arbitrarily well, both for discrete SGNN and continuous SGNN. Next, we show that such IGN belongs to \smallIGN{}.
We need the following simple assumption to ensure the input lies in a compact domain.
    \begin{assumption}\label{as:x-compact}
    There exists an upper bound on $\|x\|_{L_\infty}$ for the discrete case and $\|X\|_{L_\infty}$ in the continuous case.
    \end{assumption}

    \begin{assumption}\label{as:min-deg-lower-bound}
    $\text{min}(\degmean) \geq c_{\tn{min}}$ where $\degmean$ is defined to be $\frac{1}{n}\dg{A\one}$. The same lower bound holds for graphon case.
    \end{assumption}

\begin{restatable}[]{lemma}{ignmultiplication}
\label{lem:ign-multiplication}
Assume AS1-AS6 and $DMD$ arbitrarily well in $L_\infty$ sense on a compact domain %
\end{restatable}

    \begin{proof}
    Given diagonal matrix $D$ and matrix $M$,
    to implement $DMD$ with linear equivariant layers of \IGN, we first use operation 14-15 in \Cref{tab:R2-R2} to copy diagonal elements in $D$ to rows and columns of two matrix $D_\tn{row}$ and $D_\tn{col}$. Then calculating $DMD$ becomes entry-wise multiplication of three matrix $D_\tn{row}, M, D_\tn{col}$. Assuming all entries of $D$ and $M$ lies in a compact domain, we can use MLP (which is part of IGN according to \Cref{remark:IGN}) to approximate multiplication arbitrarily well \citep{cybenko1989approximation,hornik1989multilayer}.  for illustration.

    To implement $\frac{1}{n}Mx$ with linear equivariant layers of \IGN,
    first map $x$ into a diagonal matrix $\text{Diag}(x)$ and concatenate it with $M$ as the input
    $[\text{Diag}(x), M] \in \mb{R}^{n\times n \times 2}$ to \IGN. Apply ``copy diagonal to all columns'' to the first channel and use MLP to uniformly approximates up to arbitrary precision $\epsilon$ the multiplication of first channel with the second channel. Then use operation ``copy row mean'' to map $\mb{R}^{n\times n } \rightarrow \mb{R}^{n}$ to get the $\frac{1}{n}Mx$ within $\epsilon$ precision. See \Cref{fig:approx}.
    \end{proof}

    \begin{remark}\label{remark:IGN-mm-multiplication}
    Linear layers in \IGN{} can not implement matrix-matrix multiplication in general. When we introduce the matrix multiplication component, the expressive power of GNN in terms of WL test provably increases from 2-WL to 3-WL \citep{maron2019provably}). 
    \end{remark}

    \begin{restatable}[]{theorem}{ignapproxsgnn}
    \label{thm:ign-approx-sgnn}
    Given $n$, $\epsilon$, and $\tn{SGNN}_{\theta_1}(n)$, there exists a \tn{\IGN{}} $\tn{IGN}_{\theta_2}(n)$ such that it approximates $\tn{SGNN}_{\theta_1}(n)$ on a compact set (support of input feature $x_n$) arbitrarily well in $L_\infty$ sense. %
    \end{restatable}

\begin{figure*}[htp!]
  \centering
  \includegraphics[width=.8\linewidth]{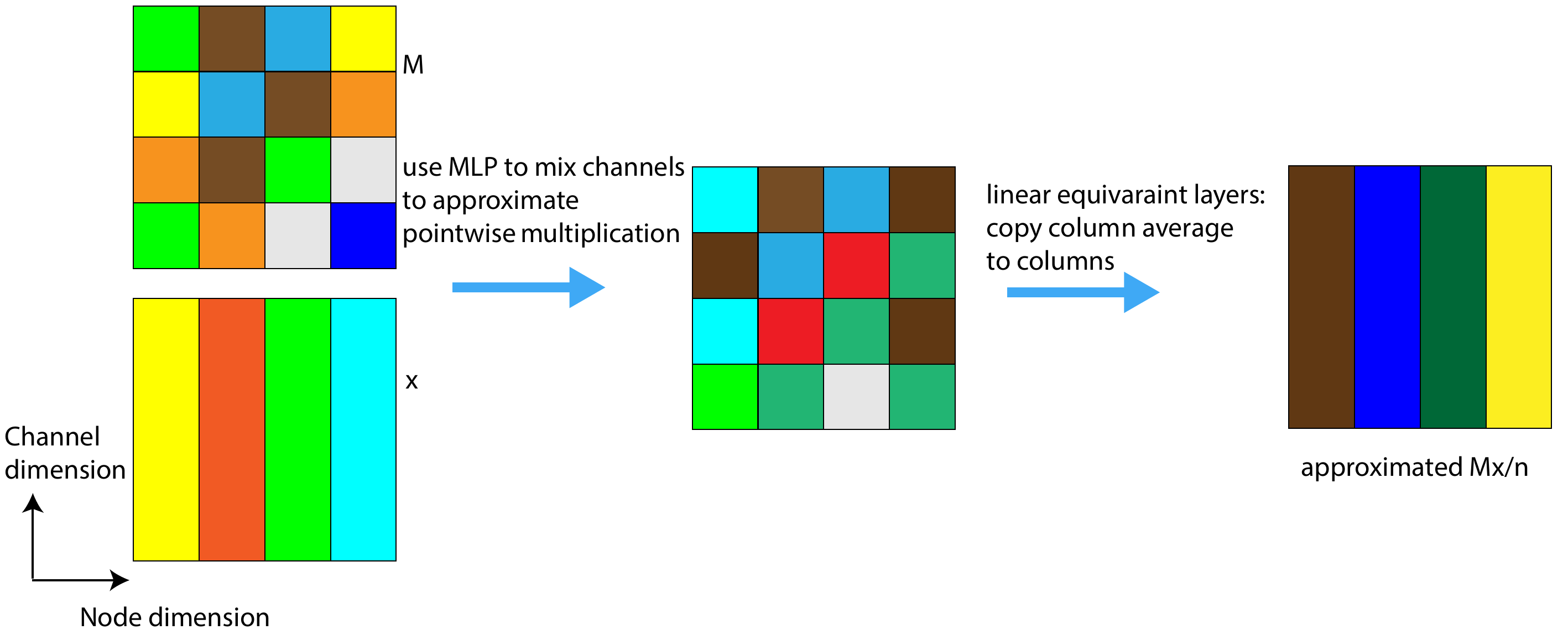}
\caption{An illustration of how we approximate the major building blocks of SGNN: $\frac{1}{n}Ax$. }
  \label{fig:approx}
\end{figure*}

    \begin{proof}
    Since IGN and SGNN has the same non-linearity. To show that IGN can approximate SGNN, it suffices to show that IGN can approximate linear layer of SGNN, which further boils down to prove that IGN can approximate $Lx$.

    Here we assume the input of \IGN{} is $A\in \mb{R}^{n \times n}$ and $x\in \mb{R}^{n \times d}$. We need to first show how $L = D(A)^{-\frac{1}{2}} A D(A)^{-\frac{1}{2}}$ can be implemented by linear layers of IGN. This is achieved by noting that $L = \frac{1}{n} \degmean^{-\frac{1}{2}} A \degmean^{-\frac{1}{2}}$ where $\degmean$ is normalized degree matrix $\frac{1}{n}\dg{A\one}$.
    Representing $L$ as $\frac{1}{n} \degmean^{-\frac{1}{2}} A \degmean^{-\frac{1}{2}}$ ensures that all entries in $A$ and $\degmean$ lies in a compact domain, which is crucial when we extending the approximation proof to the graphon case.

    Now we show how $Lx = \frac{1}{n} \degmean^{-\frac{1}{2}} A \degmean^{-\frac{1}{2}} x$ is implemented.
    First, it is easy to see that \IGN{} can calculate exactly $\degmean$ using equivariant layers. Second, as approximating a) $f(a, b) = ab$ and b) $f(a) = \frac{1}{\sqrt{a}}$ can achieved by MLP on compact domain, approximating $\degmean^{-\frac{1}{2}} A \degmean^{-\frac{1}{2}}$ can also achieved by \IGN{} layers according to \Cref{lem:ign-multiplication}. Third, we need to show $\frac{1}{n} \degmean^{-\frac{1}{2}} A \degmean^{-\frac{1}{2}}x$ can also be implemented. This is proved in \Cref{lem:ign-multiplication}.

    There are two main functions we need to approximate with MLP:
    a) $f(x) = 1/\sqrt{a}$ and b) $ f(a, b) = ab$. 

    For a) the input is entries of $\degmean$ which lie in $[0,1]$. By classical universal approximation theorem \citep{cybenko1989approximation,hornik1989multilayer}, we know MLP can approximate a) arbitrarily well.

    For b) the input is $(\degmean^{-1/2}, A)$ for normalized adjacency matrix calculation, and $(L, x)$ for graph signal convolution.
    
    To ensure the uniform approximation, we need to ensure all of them lie in a compact domain. This is indeed the case as all entries in $\degmean, A, x$ are all upper bounded

    \begin{enumerate}
    \item  every entry in $A$ is either 0 or 1 therefore lies in a compact domain.
    \item similarly, all entries $\degmean$ lies in $[c_\text{min}, 1]$ by AS\ref{as:min-deg-lower-bound}, and therefore $\degmean^{-\frac{1}{2}}$ also lies in a compact domain. As $L(A)$ is the multiplication of $\degmean^{-1/2}, A, \degmean^{-1/2}$, every  entry of $L(A)$ also lies in compact domain.
    \item input signal $x$ has bounded $l_\infty$-norm by assumption AS\ref{as:x-compact}.
    \item  all coefficient for operators is upper bounded and independent from $n$ by AS\ref{as:filter-bound}.
    \end{enumerate}

    Since we showed the $L(A)x$ can be approximated arbitrarily well by IGN, repeating such processes and leveraging the fact that $L$ has bounded spectral norm, we can then approximate $L^k(A)x$ up to $\epsilon$ precision. The errors $\epsilon$ depend on the approximation error of the MLP to the relevant function, the previous errors, and uniform bounds as well as uniform continuity of the approximated functions.
    \end{proof}

    \begin{restatable}[]{theorem}{cignapproxcsgnn}
    \label{thm:cign-approx-csgnn}
    Given $\epsilon$, and a \sGNN{} c$\tn{SGNN}_{\theta_1}$, there exists a continuous \IGN{} c$\tn{IGN}_{\theta_2} $such that it approximates c$\tn{SGNN}_{\theta_1}$ on a compact set (input feature $X$) arbitrarily well.
    \end{restatable}

    \begin{proof}

In the continuous case, $Lx = \frac{1}{n} \degmean^{-\frac{1}{2}} A \degmean^{-\frac{1}{2}} x$ in the discrete case will be replaced with $D(W)^{-\frac{1}{2}} W D(W)^{-\frac{1}{2}} X$ where $D(W)$ is a diagonal graphon defined to be $D(W)(i, i) = \int_0^1 W(i, j)dj$. 

We show that all items listed in proof of \Cref{thm:ign-approx-sgnn} still holds in the continuous case
    \begin{itemize}
      \item we consider the $W$ instead in the continuous case, where all entries still lies in a compact domain $[0,1]$.
      \item similarly all entries of the continuous analog of $\degmean, \degmean^{-\frac{1}{2}}$, and $T(W)$ also lies in a compact domain according to AS\ref{as:min-deg-lower-bound}.
      \item the statements about input signal $X$ and the coefficient for linear equivariant operators also holds in the continuous setting.
    \end{itemize}
    Therefore we conclude the proof. Now we are ready to prove that those IGN that can approximate SGNN well is a subset of \smallIGN{}.
    \end{proof}

\begin{lemma}
\label{lemma:sgnn-linear-layer-commutes}
With slight abuse of notation, let $\inducedEW$ be graphon of chessboard pattern. Let $\inducedEX$ be a graphon signal with 1D chessboard pattern.
$S_n \circ \inducedEW \inducedEX = (S_n \inducedEW) (S_n \inducedEX)$. %
\end{lemma}
\begin{proof}
 Since $ S_n \circ \inducedEW \inducedEX = S_n \circ\int_{j\in [0, 1]} \inducedEW(i, j) \inducedEX(j) dj = \left(..., \frac{1}{\sqrt{n}}\int_{j\in [0, 1]} \inducedEW(\frac{i}{n}, j) \inducedEX(j), ...\right)$, it suffices to analyze $i$-th component $\frac{1}{\sqrt{n}}\int_{j\in [0, 1]} \inducedEW(\frac{i}{n}, j) \inducedEX(j)$.

 Since $\inducedEW, \inducedEX$ are of chessboard pattern, we can replace integral with summation.
 \begin{align*}
 S_n \circ \inducedEW \inducedEX(i)
 & =  \frac{1}{\sqrt{n}}\int_{j\in [0, 1]} \inducedEW(\frac{i}{n}, j) \inducedEX(j) \\
  & = \frac{1}{\sqrt{n}} \frac{1}{n} \sum_{j \in [n]} \inducedEW(\frac{i}{n}, \frac{j}{n}) \inducedEX(\frac{j}{n}) \\
  & = \sum_{j \in [n]} \frac{1}{n} \inducedEW(\frac{i}{n}, \frac{j}{n}) ( S_n \inducedEX ) (j) \\
  & = \sum (S_n \inducedEW )(i, j) ( S_n \inducedEX ) (j) \\
  & = \left((S_n \inducedEW ) ( S_n \inducedEX ) \right)(i)
\end{align*}
Which concludes the proof. Note that our proof does make use of the property of multiplication between two numbers.
\end{proof}
\begin{remark}
\label{remark: replace-multiplication}
The whole proof only relies on that $\inducedEW$ and $\inducedEX$ have checkerboard patterns. Therefore replacing the multiplication with other operations (such as a MLP) will still hold.
\end{remark}

\smalligngcn*
\begin{proof}
To prove this, we only need to show that $S_n\Phi_{c, \tn{approx}}([\inducedEW, \inducedEf]) = \Phi_{d, \tn{approx}} S_n([\inducedEW, \inducedEf])$. Here  $\Phi_{c, \tn{approx}}$ and $\Phi_{d, \tn{approx}}$ denotes those specific IGN in \Cref{thm:ign-approx-sgnn,thm:cign-approx-csgnn} constructed to approximate SGNN. 

To build up some intuition, let $\Phi_{\tn{SGNN}}$ denotes the \sGNN{} that $\Phi_{\tn{approx}}$ approximates. it is easy to see that $S_n\Phi_{c, \tn{SGNN}}([\inducedEW, \inducedEf]) = \Phi_{d, \tn{SGNN}}S_n([\inducedEW, \inducedEf])$  due to \Cref{lemma:sgnn-linear-layer-commutes} and \Cref{lem:property-of-T}.2. 
To show the same holds for $\Phi_{\tn{approx}}$, note that the only difference between $\inducedEW\inducedEf$ implemented by SGNN and approximated by  $\Phi_{\tn{approx}}$ is that $\Phi_{\tn{approx}}$ use MLP to simulate multiplication between numbers. According to \Cref{remark: replace-multiplication}, the approximated version of $\inducedEW\inducedEf$ still commutes with $S_n$.    %

Since nonlinear layer $\sigma$ in $\Phi_{\tn{approx}}$ also commutes with $S_n$ according to \Cref{lem:property-of-T}.2, we can combine the result above and conclude that $\Phi_{\tn{approx}}$ commutes with $S_n$.  Therefore $\Phi_{\tn{approx}}$ belongs to \smallIGN{}, which finishes the proof.
\end{proof}

\section{More experiments}
\label{app:all-exps}

We next show full results to verify \Cref{thm:convergenceafterEM,thm:EW-convergence}. The main procedure is described in \Cref{sec:exp}.

As the ground truth is defined in the continuous regime, we use outputs of IGN on large graphs as the approximation of the unknown true limit. We experiment with two methods: a) we take the output of IGN from the deterministic edge weight continuous model as ground truth and b) we take graphs sampled from the stochastic edge weight continuous model as input to IGN and average the outputs over 10 random seeds. The case a) is shown In the main text. Here we include results for both a) and b).

\begin{figure*}[htp]

\begin{subfigure}%
  \centering
  \includegraphics[width=.24\linewidth]{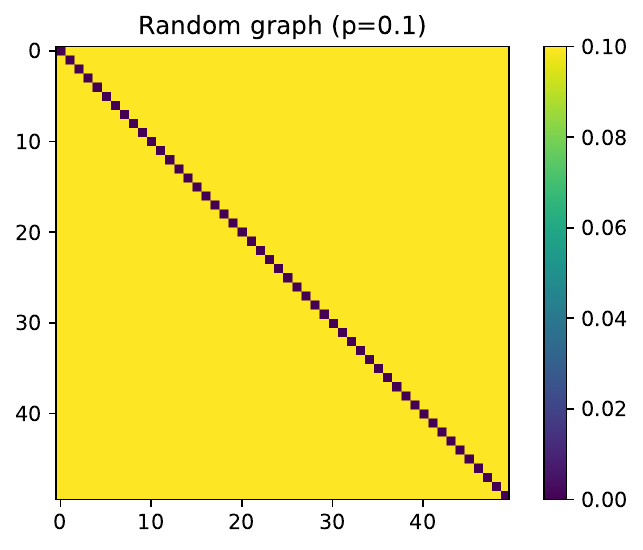}
  \includegraphics[width=.24\linewidth]{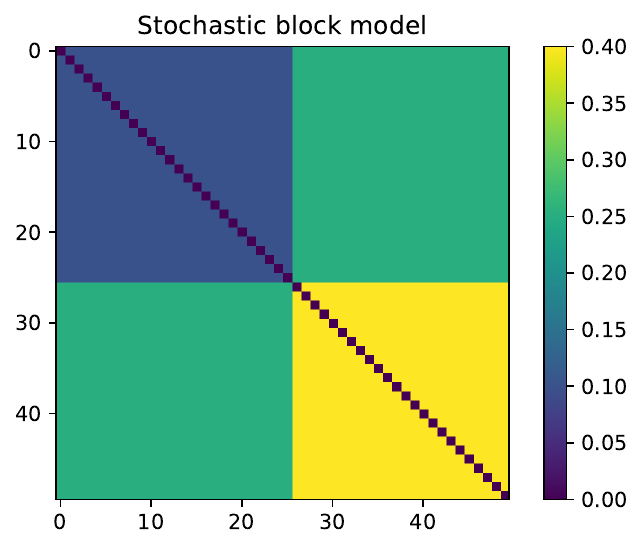}
  \includegraphics[width=.24\linewidth]{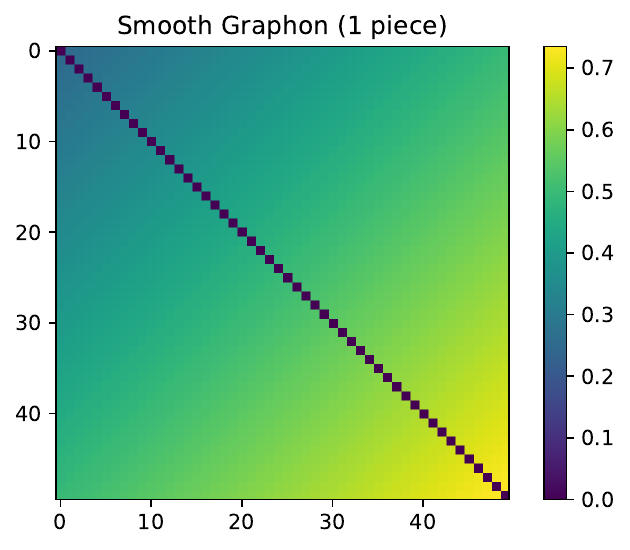}
  \includegraphics[width=.24\linewidth]{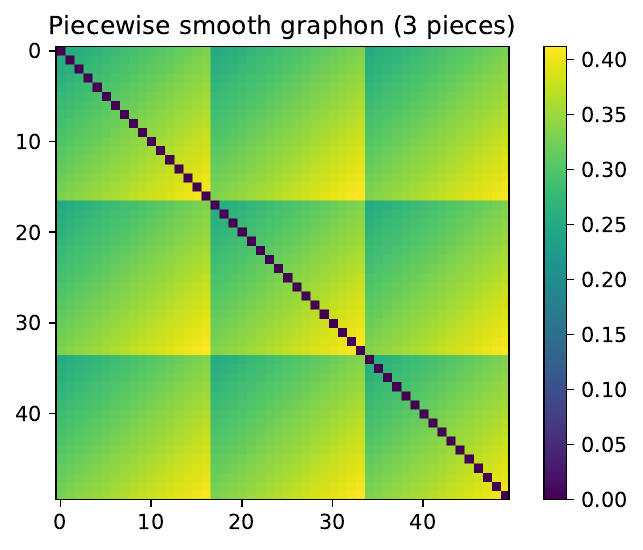}
  \caption{Four graphons of increasing complexity.}
  \label{fig:sfig0-full}
\end{subfigure}

\begin{subfigure}%
  \centering
  \includegraphics[width=.24\linewidth]{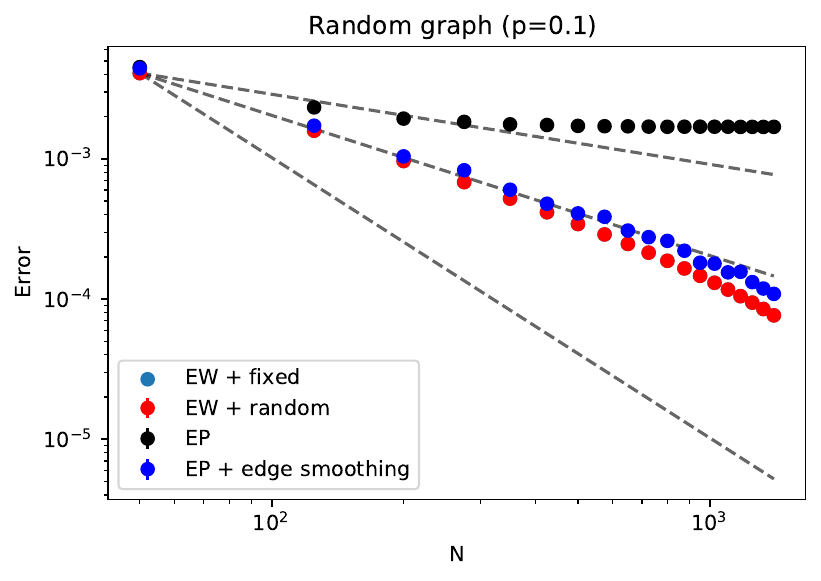}
  \includegraphics[width=.24\linewidth]{fig/sbm_78_groundtruth_randomsample.pdf}
  \includegraphics[width=.24\linewidth]{fig/graphon_78_groundtruth_randomsample.pdf}
  \includegraphics[width=.24\linewidth]{fig/graphonPW3_78_groundtruth_randomsample.pdf}
  \caption{ground truth: random sample. }
  \label{fig:sfig1-full}
\end{subfigure}

\begin{subfigure}%
  \centering
  \includegraphics[width=.24\linewidth]{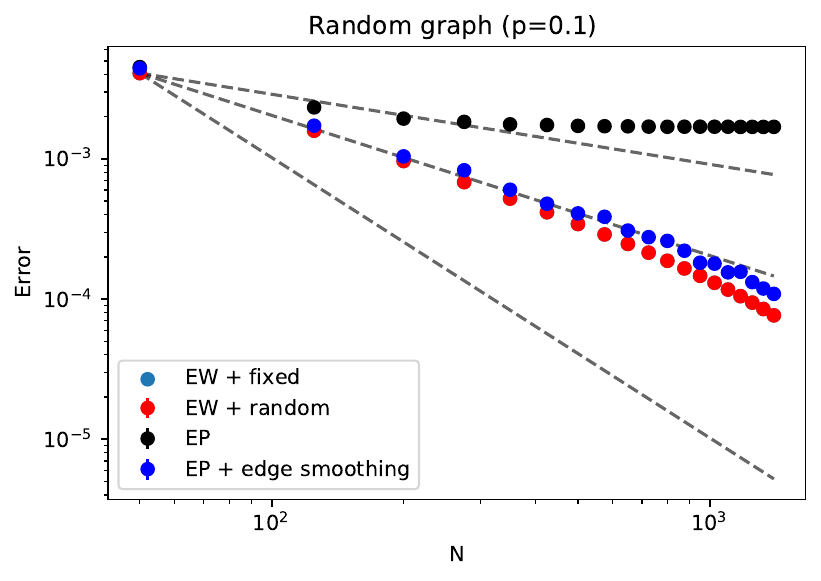}
  \includegraphics[width=.24\linewidth]{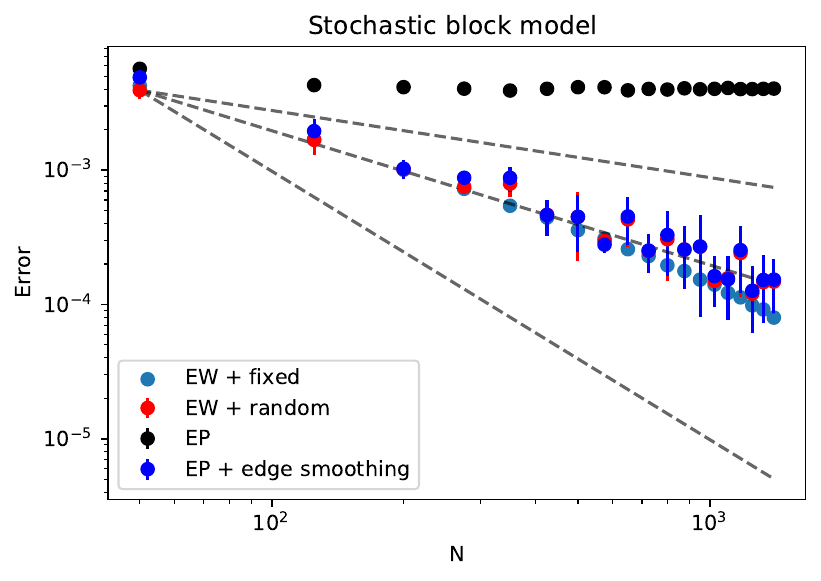}
  \includegraphics[width=.24\linewidth]{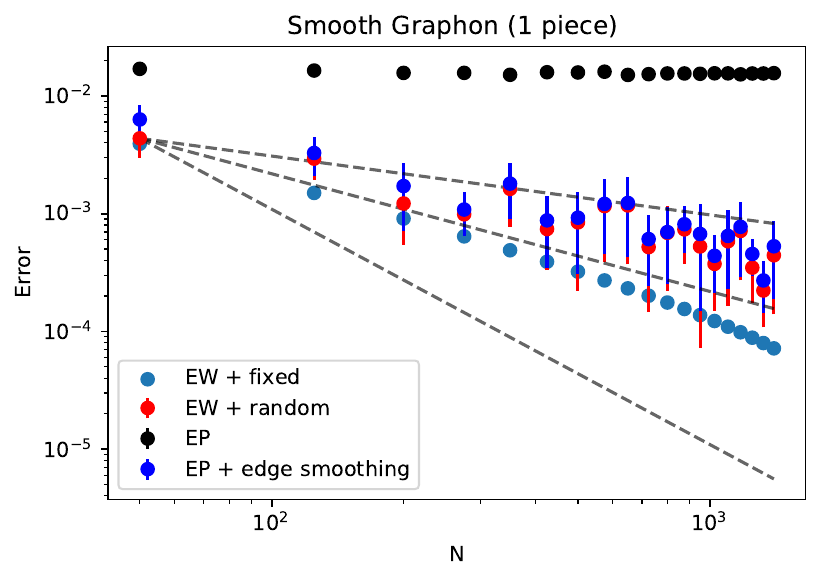}
  \includegraphics[width=.24\linewidth]{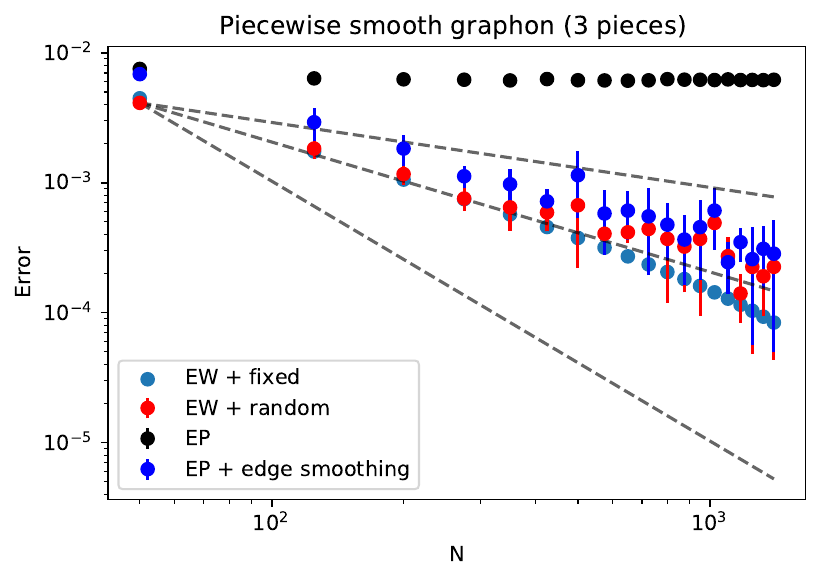}
  \caption{ground truth: grid sample.}
  \label{fig:sfig2-full}
\end{subfigure}
\caption{The convergence error for four generative models under two ways of approximating ground truth. Three dashed line in each figure indicates the decay rate of $n^{-0.5}, n^{-1}$ and $n^{-2}$. EW stands for edge weight continuous model and EP stands for edge probability discrete model. As implied by \Cref{thm:EW-convergence}, EW + fixed and EW + random both converges when $n$ increases. On the other hand, EP does not converge, which is consistent with \Cref{thm:conv-failure}. After edge probability estimation, EP + edge smoothing again converges,  which is consistent with \Cref{thm:convergenceafterEM}.}
\label{fig:convergence-full}
\end{figure*}

\section{Third-party results}
\label{app:third-party}

\subsection{Edge Probability Estimation from \citet{zhang2015estimating}}
\label{subsec:zhang}
We next restate the setting and theorem regarding the theoretical guarantee of the edge probability estimation algorithm.

\begin{definition}
For any $\delta, A_1>0$, let $\mathcal{F}_{\delta ; L}$ de note a family of piecewise Lipschitz graphon functions $f:[0,1]^{2} \rightarrow[0,1]$ such that $(i)$ there exists an integer $K \geq 1$ and a sequence $0=x_{0}<\cdots<x_{K}=1$ satisfying $\min _{0 \leqslant  s \leqslant  K-1}\left(x_{s+1}-\right.$ $\left.x_{s}\right) \geq \delta$, and (ii) both $\left|f\left(u_{1}, v\right)-f\left(u_{2}, v\right)\right| \leqslant  A_1\left|u_{1}-u_{2}\right|$ and $\left|f\left(u, v_{1}\right)-f\left(u, v_{2}\right)\right| \leqslant  A_1 \mid v_{1}-$ $v_{2} \mid$ hold for all $u, u_{1}, u_{2} \in\left[x_{s}, x_{s+1}\right], v, v_{1}, v_{2} \in\left[x_{t}, x_{t+1}\right]$ and $0 \leqslant  s, t \leqslant  K-1$
\end{definition}

Assume that $\alpha_n$ = 1. It is easy to see that the setup considered in \citet{zhang2015estimating} is slightly more general than the setup in \citet{keriven2020convergence}. The statistical guarantee of the edge smoothing algorithm is stated below.

\begin{theorem}[\citet{zhang2015estimating}]\label{thm:graphon-estimation}
Assume that $A_1$ is a global constant and $\delta=\delta(n)$ depends on $n$, satisfying $\lim_{n \rightarrow \infty} \delta /(n^{-1} \log n)^{1 / 2} \rightarrow \infty$. Then the estimator $\tilde{P}$ with neighborhood $\mathcal{N}_{i}$ defined in \citet{zhang2015estimating} and $h=C(n^{-1} \log n)^{1 / 2}$ for any global constant $C \in(0,1]$, satisfies
$
\max_{f \in \mathcal{F}_{\delta; A_1}} \operatorname{pr}\{d_{2, \infty}(\tilde{P}, P)^{2} \geq C_{1}(\frac{\log n}{n})^{1 / 2}\} \leqslant  n^{-C_{2}}
$
where $C_{1}$ and $C_{2}$ are positive global constants. Here, $d_{2, \infty}(P, Q) \coloneqq n^{-1 / 2}\|P-Q\|_{2, \infty}=\max_{i} n^{-1 / 2} \|P_{i}-Q_{i} \|_{2}$.
\end{theorem}

\subsection{IGN Details}
\label{subsec:IGN-details}

\begin{remark}[independence from $n$]
Although for large $n$, the result in \citet{maron2018invariant} is correct. But as noted by \citet{finzi2021practical}, this does not hold when $n$ is small, which is not an issue as we consider cases when $n$ goes to infinity in this paper. %
\end{remark}

\begin{remark}[multi-channel IGN contains MLP]
\label{remark:IGN}
For simplicity, in the main text, we focus on the case when the input and output tensor channel number is 1. The general case of multiple input and output channels is presented in Equation 9 of \citet{maron2018invariant}. The main takeaway is that permutation equivariance does not constrain the mixing over feature channels, i.e., the space of linear equivariant maps from $\mb{R}^{n^{\ell}\times d_1} \rightarrow \mb{R}^{n^{m}\times d_2}$ if of dimension $d_1d_2\bell{l+m}$. Therefore IGN contains MLP.  %
\end{remark}

\end{document}